\documentclass{article}
\usepackage{iclr2027_conference,times}
\usepackage[T1]{fontenc}
\usepackage[utf8]{inputenc}
\usepackage{amsmath,amssymb,amsthm,mathtools,bm}
\usepackage{booktabs,array,multirow,tabularx}
\usepackage{graphicx,microtype}
\usepackage{float}
\usepackage{xcolor}
\usepackage{enumitem}
\usepackage[hidelinks]{hyperref}
\usepackage[nameinlink,noabbrev,capitalise]{cleveref}
\usepackage{preprint_layout}

\usepackage{aliascnt}
\newtheorem{theorem}{Theorem}
\newaliascnt{proposition}{theorem}
\newtheorem{proposition}[proposition]{Proposition}
\aliascntresetthe{proposition}
\newaliascnt{corollary}{theorem}
\newtheorem{corollary}[corollary]{Corollary}
\aliascntresetthe{corollary}
\newaliascnt{lemma}{theorem}
\newtheorem{lemma}[lemma]{Lemma}
\aliascntresetthe{lemma}
\newaliascnt{assumption}{theorem}

\aliascntresetthe{assumption}
\newaliascnt{remark}{theorem}

\aliascntresetthe{remark}
\crefname{proposition}{Proposition}{Propositions}
\Crefname{proposition}{Proposition}{Propositions}
\crefname{corollary}{Corollary}{Corollaries}
\Crefname{corollary}{Corollary}{Corollaries}
\crefname{lemma}{Lemma}{Lemmas}
\Crefname{lemma}{Lemma}{Lemmas}
\crefname{assumption}{Assumption}{Assumptions}
\Crefname{assumption}{Assumption}{Assumptions}
\crefname{remark}{Remark}{Remarks}
\Crefname{remark}{Remark}{Remarks}

\newcommand{\g}{\mathfrak g}
\newcommand{\h}{\mathfrak h}
\newcommand{\R}{\mathbb R}
\newcommand{\E}{\mathbb E}
\newcommand{\dd}{\mathrm d}

\newcommand{\norm}[1]{\lVert #1\rVert}

\newcommand{\diag}{\operatorname{diag}}
\newcommand{\tr}{\operatorname{tr}}
\newcommand{\rank}{\operatorname{rank}}
\newcommand{\range}{\operatorname{range}}

\title{Symmetry, Defects, and Diffusion in Continuous-Memory Recurrent Networks}
\newcommand{\PreprintAuthors}{Hanson Hanxuan Mo}
\newcommand{\PreprintAuthorsPlain}{Hanson Hanxuan Mo}
\newcommand{\PreprintAffiliations}{}
\newcommand{\PreprintContact}{}

\author{\PreprintAuthors}
\date{8 September 2026}
\hypersetup{
  pdftitle={Symmetry, Defects, and Diffusion in Continuous-Memory Recurrent Networks},
  pdfauthor={\PreprintAuthorsPlain},
  pdfsubject={Preprint},
  pdfkeywords={equivariant recurrent networks, continuous memory, dynamical systems, stochastic geometry}
}

\begin{document}
\maketitle

\begin{abstract}
Continuous-memory recurrent networks must preserve phase, position, or
orientation despite model imperfections and state noise. We develop a geometric
framework that separates three questions: how many memory coordinates are
neutrally transported, how deterministic perturbations alter their
finite-horizon stability, and how ambient noise is decoded along them. Exact
per-input equivariance transports analytical group tangents pathwise and, on a
compact nondegenerate orbit stratum, yields at least $q=\dim(G/H)$ zero
group-tangent Lyapunov exponents under stationary ergodic driving. For imperfect
dynamics, a four-block tangent/normal decomposition gives local and
finite-horizon bounds on tangent growth and subspace rotation, distinguishing
first-order direct damage from second-order leakage through contracting normal
directions. For noisy dynamics, a specified decoder maps ambient covariance
$Q$ to coordinate covariance $ZQZ^\top$; under isotropic noise and fixed tangent
energy, least-squares decoding and scaled-isometric action geometry
minimize local diffusion. Local and finite-horizon evaluations include cases
both within and outside the sufficient conditions. In a fresh twenty-seed $T^2$ replication, a
decoder-covariance objective improves noisy horizon-256 memory in every pair
while meeting a prespecified clean-error equivalence margin. Direct noise
training also improves noisy memory but incurs a clean-error tradeoff. In
coupled $T^4/T^8$ integrators, the advantage of anisotropic-covariance over
isotropic regularization reverses when evaluation noise becomes isotropic. These
results connect continuous symmetry to measurable limits and design choices
for recurrent memory under explicit dynamical, decoder, and noise assumptions.
\end{abstract}

\section{Introduction}
An RNN that stores heading, position, phase, or another continuous variable
must preserve distinctions among a continuum of recurrent states over long
horizons \citep{chaudhuri2016memory}. Classical continuous-attractor models obtain this behavior by making
the encoded direction neutral while stabilizing transverse perturbations
\citep{seung1996,seung1998,burakFiete2009}; trained navigation networks can
also acquire path-integrating and grid-like representations
\citep{cuevaWei2018,banino2018,sorscher2023grid}. In practice, however, a
continuous memory can fail in three different ways: it may have no protected
neutral coordinate, a deterministic implementation defect may produce
forgetting, pinning, or drift, and state noise may cause diffusion even when
the deterministic dynamics is exactly neutral
\citep{burakFiete2012,kilpatrick2013wandering,agmon2023multiple}.

We study these as progressively more realistic versions of the same recurrent
memory problem:
\begin{enumerate}[leftmargin=1.5em,itemsep=0.1em,topsep=0.2em]
\item \textbf{Ideal memory:} how many continuous memory coordinates are
protected against exponential forgetting or amplification?
\item \textbf{Imperfect memory:} after exact protection is broken, how does the
location of the defect determine local timescales, mode rotation, pinning, or
drift?
\item \textbf{Noisy memory:} under a specified state-noise and decoder model,
how does representation geometry determine local coordinate diffusion?
\end{enumerate}
A numerically small Lyapunov exponent alone answers none of these questions.
The common geometric object is instead the infinitesimal action $A_x$, whose
image is the analytical tangent space of the represented memory orbit. Exact
equivariance transports this image without exponential distortion. A reducing
frame then exposes how defects route between memory-tangent and contracting
normal coordinates. Finally, a reduced quotient-coordinate action
$\bar A_x$ and a local memory decoder project state noise back to the represented
variable, making orbit geometry operational rather than decorative.

Our first statement is a recurrent-cocycle specialization of a classical
consequence of equivariant dynamics, not a claim that the equivariant-flow
identity is new \citep{golubitsky1988,krupa1990,rumberger2001}. The contribution
is to turn that identity into a predictive memory-stability framework: a
stabilizer-aware count of protected coordinates, a quantitative account of how
local defects couple into those coordinates, and a decoder-aware noise floor.
We distinguish structural validation from task-level consequences. Exact equivariance, direct
tangent transport, orbit dimension, and subspace alignment test structural
claims; path-integration error tests their functional consequences rather than
proving the mathematics.

Closest recurrent work connects continuous attractors to translation
equivariance and Lie generators, develops symmetry-shaped manifolds
\citep{zhang2022translation,diBernardo2025}, and studies flow-equivariant
sequence maps \citep{keller2025fernn}; robust memory
manifolds can also arise without exact symmetry \citep{can2025robust}.
Control-input realizations of temporal-scaling group operators provide a nearby
but distinct sequence-generation construction \citep{zuo2023temporal}.
Goldstone-like propagation and conformal or isometric navigation are studied in
\citep{iqbal2026goldstone,xu2023lie,xu2025grid}. Our distinct contribution is
the joint stabilizer-aware Lyapunov, four-block defect, and decoder-covariance
framework; the underlying identities, generalized least squares, and
tight-frame algebra are not claimed as new.

Orbit rank identifies protected dimension, measured tangent/normal blocks
yield local or trajectory-conditioned finite-horizon bounds, and the
action--decoder pair predicts local coordinate diffusion.

We evaluate deterministic perturbation bounds on trained $S^1/T^2$
integrators and noise-limited precision on prescribed and learned $T^2/T^4/T^8$
geometries. A noise-switch control asks whether a covariance objective selects
geometry suited to the declared noise rather than universally preferred
Euclidean isotropy. These conditional interventions do not imply superiority
over unconstrained recurrent architectures.

\paragraph{Contributions.}
We (i) state a per-input controlled-map theorem that counts protected
continuous-memory coordinates and recovers the autonomous-flow result as a
corollary; (ii) derive computable local and finite-horizon perturbation bounds
that separate direct memory-tangent damage from leakage through contracting
transverse dynamics;
(iii) include the decoder in the stochastic analysis, proving a least-squares
noise floor and a scaled-isometric optimum under explicit assumptions; and
(iv) test these predictions with interventions fixed before evaluation on controlled and
learned recurrent systems.

Selected finite-horizon and architecture-specific results are machine-checked
in Lean 4 \citep{deMouraUllrich2021,mathlib2020}; the supplement maps
manuscript claims to declarations.

\begin{table}[t]
\centering
\footnotesize
\setlength{\tabcolsep}{4pt}
\caption{Overview of the questions, measurable predictors, evidence, and
assumptions. Formal statements are conditional on their listed scope; task
behavior is consequence evidence rather than theorem proof.}
\label{tab:overview}
\begin{tabularx}{\linewidth}{@{}>{\raggedright\arraybackslash}p{0.16\linewidth}>{\raggedright\arraybackslash}p{0.24\linewidth}>{\raggedright\arraybackslash}p{0.25\linewidth}>{\raggedright\arraybackslash}X@{}}
\toprule
Question & Predictor & Evidence & Scope\\
\midrule
How many neutral coordinates? & Orbit rank $q=\dim(G/H)$ and exact tangent
transport & Count, direct-tangent, and alignment checks & Exact symmetry and a
nondegenerate stratum are required\\
How do defects act? & Normal gap, frame conditioning, and four local defect
blocks & Controlled and native-coordinate interventions & Sufficient local
repeated-map and finite-horizon bounds; the latter are offline and trajectory-conditioned\\
How does noise enter memory? & $D_\psi=ZQZ^\top$ and reduced action geometry &
Paired prescribed and learned $T^2/T^4/T^8$ geometry & Local decoder/noise/resource model
is explicit\\
\bottomrule
\end{tabularx}
\end{table}

\section{Ideal continuous memory: protected dimensionality}
\label{sec:exact}
Within the symmetry-orbit construction studied here, preserving $q$
independent continuous coordinates without exponential forgetting or
amplification requires $q$ neutral tangent directions. A continuous group orbit
supplies such candidate memory coordinates, while exact equivariance makes
their transport analytically identifiable rather than inferred from a
finite-time spectrum.

Consider a controlled recurrent update
\begin{equation}
    x_{t+1}=F_{u_t}(x_t), \qquad x_t\in M,
    \label{eq:controlled_update}
\end{equation}
where a Lie group $G$ acts smoothly on the finite-dimensional state manifold
$M$. The controls used here are group invariant; per-input equivariance means
\begin{equation}
    F_u(g\!\cdot\!x)=g\!\cdot\!F_u(x)
    \quad\text{for every admissible }u,g,x.
    \label{eq:per_input_equivariance}
\end{equation}
For $\xi\in\g$, let
$\xi_M(x)=\left.\frac{\dd}{\dd s}\right|_{0}\exp(s\xi)\cdot x$ and define the
infinitesimal action $A_x\xi=\xi_M(x)$. Its image
$E_x^G=T_x(G\cdot x)$ is the group-tangent space. On a fixed orbit stratum
with isotropy Lie algebra $\h$, the protected orbit dimension is
$q=\dim(\g/\h)=\dim(G/H)$. This is the number of protected continuous
coordinates supplied by the group action, not a general information- or
parameter-count statement.
For orbit-type statements, assume $G$ is compact, or more generally that its
action on $M$ is proper. Stabilizers in one orbit-type stratum are conjugate;
they need not be the same embedded subgroup \citep{palais1961slices}.

\begin{theorem}[Exact tangent transport for controlled recurrent maps]
\label{thm:controlled_neutral}
Let every $F_u$ be $C^1$, and let $x_t$ be a trajectory of
\cref{eq:controlled_update} contained in a compact set $K$. Assume
\cref{eq:per_input_equivariance}, $\rank A_x=q$ on $K$, and
that the nonzero singular values of $A_x$ lie in $[c,C]$ with
$0<c\le C<\infty$. If $K$ lies in an orbit-type stratum with representative
stabilizer $H$, then $q=\dim(G/H)$. For every $m>n$,
\begin{equation}
 \begin{aligned}
 &D F_{u_{m-1}}(x_{m-1})\cdots D F_{u_n}(x_n)\\
 &\hspace{5em}{}\cdot A_{x_n}=A_{x_m}.
 \end{aligned}
 \label{eq:exact_tangent_transport}
\end{equation}
Consequently, for every $v\in E^G_{x_n}$,
\[
 \frac{c}{C}\norm v\le \norm{P_{m:n}v}\le\frac{C}{c}\norm v,
\]
where $P_{m:n}$ denotes the displayed tangent product. All $q$ restricted
singular-value growth rates are therefore zero, and every nonzero protected
vector has zero direct logarithmic growth rate whenever the corresponding
limit is defined. For a prescribed nonstationary input, the conclusion is
this exact pathwise bound; no stationary spectrum is asserted without an
additional driving measure.
\end{theorem}

The transport identity follows by one differentiation: differentiating
$F_u(\exp(s\xi)\cdot x)=\exp(s\xi)\cdot F_u(x)$ at $s=0$ gives
$D F_u(x)A_x\xi=A_{F_u(x)}\xi$, and iteration gives
\cref{eq:exact_tangent_transport}. Equivariance also implies
$\ker A_{x_t}\subseteq\ker A_{x_{t+1}}$; constant rank makes these kernels
equal along the trajectory. This equality of infinitesimal kernels is what the
proof uses; it does not require all stabilizer subgroups on the stratum to be
literally equal. Uniform conditioning then bounds both expansion and
contraction. For smooth $A_x$ of constant rank on compact $K$, the stated
nonzero-singular-value bounds follow automatically for fixed smooth metrics;
we retain them explicitly because measured conditioning later controls ambient
angle bounds.

\begin{corollary}[Stationary controlled cocycle]
\label{cor:stationary_cocycle}
Let $(\Omega,\mathcal F,\nu)$ be a standard Borel probability space, let
$\vartheta$ be invertible, ergodic, and measure preserving, and set
$Y=\Omega\times K$. Suppose
$\Theta(\omega,x)=(\vartheta\omega,F_\omega(x))$ is measurable and preserves an
ergodic probability measure $\mu$ on $Y$ whose $\Omega$-marginal is $\nu$.
For the full-spectrum assertion, either assume that $\Theta$ is invertible
modulo $\mu$, or pass to the natural extension of $(Y,\mu,\Theta)$ and lift the
cocycle and group-tangent bundle through its zeroth-coordinate projection.
Assume the derivative cocycle
$L(\omega,x)=D F_\omega(x)$ and $E_x^G$ are measurable and
\begin{equation}
 \int_{\Omega\times K}\log^+\!\norm{L(\omega,x)}\,\dd\mu(\omega,x)<\infty.
 \label{eq:met_integrability}
\end{equation}
Under the hypotheses of \cref{thm:controlled_neutral}, the restricted cocycle
on the invariant rank-$q$ bundle $E^G$ has $q$ zero Lyapunov exponents for
$\mu$-almost every base point.
On the invertible skew product or its natural extension, a measurable adapted
orthonormal frame makes the full derivative cocycle block upper triangular
with this restricted cocycle as a diagonal block. The finite-dimensional
semi-invertible multiplicative ergodic theorem therefore implies that the
full forward spectrum contains \emph{at least} $q$ zero exponents for
$\mu$-almost every base point, counted with multiplicity
\citep{oseledets1968,froyland2013semi}. Additional flow-neutral or transverse
zero exponents are not excluded.
\end{corollary}

The logical dependence is intentionally one way: exact equivariance gives an
invariant rank-$q$ group-tangent family; uniform action conditioning gives
zero direct rates on that family; and only the stationary hypotheses in
\cref{cor:stationary_cocycle} promote those rates to an at-least-$q$
full-spectrum statement. The imperfect-symmetry result below is a separate
finite-horizon bound and makes no asymptotic persistence claim.

For an autonomous $C^1$ equivariant flow $\phi_t$, flow equivariance and one
differentiation give the classical identity
\begin{equation}
 \phi_t(\exp(s\xi)\!\cdot\!x)=\exp(s\xi)\!\cdot\!\phi_t(x)
 \quad\Longrightarrow\quad
 D\phi_t(x)\xi_M(x)=\xi_M(\phi_t(x)).
 \label{eq:flow_tangent_transport_main}
\end{equation}
Restricting $\xi$ to a complement of the stabilizer kernel produces the $q$
independent orbit directions used above \citep{golubitsky1988,krupa1990} and
recovers the protected-flow theorem. This group-tangent zero block is distinct
from the ordinary time-translation zero
exponent: on a fixed-point orbit $f(x)=0$, while product groups with $q>1$
produce multiplicity beyond any single flow direction. Continuous symmetry is
essential to this mechanism; a finite group has no nonzero infinitesimal
generators.

\begin{figure}[H]
\centering
\includegraphics[width=\linewidth]{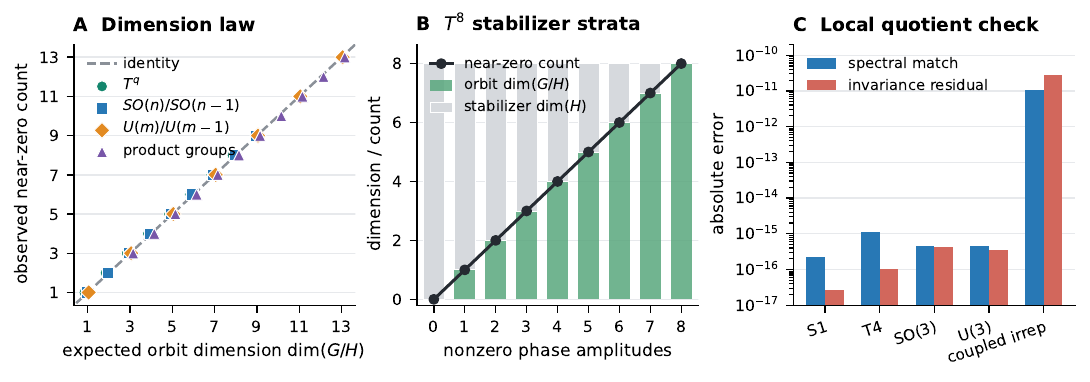}
\caption{\textbf{Ideal-memory orbit geometry.} (A) Across product tori,
$SO(n)/SO(n-1)$, $U(m)/U(m-1)$, and product-group systems, the numerical
near-zero count matches $q=\dim(G/H)$. Counts are supporting diagnostics; the
primary check is direct transport of the analytical tangents. (B) A stratified
$T^8$ control changes protected rank as phase factors become active and the
stabilizer loses dimensions. (C) A local fixed-point Jacobian decomposes into
orbit and quotient blocks with spectral and invariance residuals at numerical
precision. The coupled-irrep Jacobian is finite-differenced, explaining its
larger but still $\sim10^{-11}$ residual.}
\label{fig:exact_geometry}
\end{figure}

Figure~\ref{fig:exact_geometry} tests multiplicity, stabilizer dependence, and
local quotient removal. The quotient panel compares the full and reduced local
Jacobians at the same orbit point, rather than combining a finite-time QR
spectrum with a local quotient spectrum. The maximum local
spectral discrepancy is $1.0\times10^{-11}$ in the finite-difference
coupled-irrep case and at or below $1.1\times10^{-15}$ in the analytic-Jacobian
families. Direct equivariance, tangent exponents, principal-angle alignment,
and flow-zero controls are tabulated in \cref{app:exact_diagnostics}.

\section{Imperfect continuous memory: perturbation bounds}
\label{sec:deterministic}
Exact protection is binary, but once it is broken, memory degradation depends
on how the perturbation reaches the memory tangent. To make that statement
local and testable, select a time map of duration $\tau$ at a fixed-point orbit,
or make an explicit co-moving identification under which its group-tangent
return map is the identity, and choose reducing tangent/normal coordinates in
which
\begin{equation}
 T_0=\begin{pmatrix}I&0\\0&B\end{pmatrix},
 \qquad \norm{B}\le\rho<1.
 \label{eq:reduced_exact_map}
\end{equation}
An invariant group bundle guarantees a zero lower-left block, but an arbitrary
orthogonal complement can leave an upper-right block. A shear removes it when
$I-B$ is invertible; its condition number $\kappa_S=\norm{S}\norm{S^{-1}}$ is
reported rather than hidden (\cref{app:shear}).

For a perturbed local map, write
\begin{equation}
T=T_0+\begin{pmatrix}E_{GG}&E_{GN}\\E_{NG}&E_{NN}\end{pmatrix},
\quad
(a,b,c,e)=\bigl(\norm{E_{GG}},\norm{E_{GN}},
\norm{E_{NG}},\norm{E_{NN}}\bigr).
\label{eq:defect_blocks}
\end{equation}
Let $d=1-\rho$ be the normal gap and $D=d-a-e$.

\begin{proposition}[Blockwise invariant-graph criterion and rate bounds]
\label{prop:block_certificate}
Suppose $D>0$, $D^2\ge4bc$, and define
\begin{equation}
 r_*=\begin{cases}
 \displaystyle\frac{2c}{D+\sqrt{D^2-4bc}},&b>0,\\[0.7em]
 c/D,&b=0.
 \end{cases}
 \label{eq:rstar}
\end{equation}
If $1-a-br_*>0$, the graph transform maps the closed ball
$\{L:E^G\to E^N:\norm L\le r_*\}$ into itself and hence admits an invariant
graph in that ball. On any such graph, set $s=a+br_*$. If $s<1$, its local
memory exponents under repeated iteration of this same selected linear map
satisfy
\begin{equation}
 \frac{1}{\tau}\log(1-s)\le\lambda_G\le
 \frac{1}{\tau}\log(1+s),
 \label{eq:exp_bound_main}
\end{equation}
and its largest ambient principal angle from the exact tangent space obeys
\begin{equation}
 \sin\alpha\le\min\{1,\kappa_S r_*\}.
 \label{eq:angle_bound_main}
\end{equation}
\end{proposition}

\begin{proposition}[Finite-horizon transported-graph bounds]
\label{prop:finite_horizon_candidate}
Fix $n<m$. Along a prescribed trajectory, let one coherent frame sequence put
the exact reference maps in the form
$T_t^0=\operatorname{diag}(I,B_t)$, $\norm{B_t}\le\rho_t<1$, and write the
perturbed maps as $T_t=T_t^0+E_t$, with the four $(GG,GN,NG,NN)$ block norms of
$E_t$ bounded by $(a_t,b_t,c_t,e_t)$. Starting from a specified graph
$L_n$ with $\norm{L_n}\le r_n$, define, for $t=n,\ldots,m-1$,
\begin{equation}
 s_t=a_t+b_t r_t,\qquad
 r_{t+1}=\frac{c_t+(\rho_t+e_t)r_t}{1-a_t-b_t r_t}.
 \label{eq:finite_horizon_recursion_candidate}
\end{equation}
If every denominator is positive, the specified graph has a unique forward
image at each step, and its norm is bounded by the displayed scalar recursion:
it remains a graph with $\norm{L_t}\le r_t$ and satisfies
$T_tJ_{L_t}=J_{L_{t+1}}M_t$. Its center product lies
between $\prod_t(1-s_t)$ and $\prod_t(1+s_t)$ in conorm and norm. Ambient
bounds add only endpoint graph-chart factors, while the largest ambient
principal angle $\alpha_t$ from the transported graph to the exact center
plane obeys $\sin\alpha_t\le\min\{1,\kappa(S_t)r_t\}$. For positive step durations
$\tau_t$, taking logarithms of the product bounds, including the endpoint
factors, and dividing by $\sum_{t=n}^{m-1}\tau_t$ gives finite-horizon
physical-time rate intervals. This is finite horizon only: it gives no
asymptotic Lyapunov spectrum or bi-infinite invariant bundle.
\end{proposition}

The proof and implementation details are in \cref{app:finite_horizon_candidate}. Here
``unique'' is conditional on the specified $L_n$: the proposition neither
selects a preferred initial graph nor identifies an invariant graph shared by
all initial conditions. Its terminal
shear is future-dependent, hence offline unless the future sequence is known.
The constant-map case recovers \cref{prop:block_certificate} only when
initialized at an invariant graph supplied by Proposition~\ref{prop:block_certificate}.

The criterion is sufficient and local, not a replacement for global
normally hyperbolic invariant-manifold theory
\citep{fenichel1971,hirschPughShub1977}. Its useful feature is block
resolution. In the fixed-gap regime $d>0$ held constant as $\epsilon\to0$,
direct tangent damage $a=O(\epsilon)$ gives
$s=O(\epsilon)$. With $a=e=0$ and bidirectional leakage
$b,c=O(\epsilon)$,
\begin{equation}
 r_*=\frac{c}{d}+O(\epsilon^3/d^3),\qquad
 s=\frac{bc}{d}+O(\epsilon^4/d^3),
 \label{eq:leakage_order}
\end{equation}
so the exponent displacement is second order and inverse-gap amplified.
One-way leakage alone does not move the center exponent in this constant block
model. The full proof, physical-time scaling, and a local residual-branch
realization lemma are in \cref{app:block_proof}.

\paragraph{Selected-map perturbations and recurrent-model checks.}
Across six coupled-irrep instances and six learned exact $S^1$ checkpoints,
1,572 of 1,596 nonzero reducing-block interventions satisfy the conditions and
1,476 yield informative bounds, with no bound violations; nested rows are not
independent trials. In-family noise on 90 exact-cell parameter sets preserves
symmetry to $1.24\times10^{-16}$ tangent-transport error. Of 2,520 ambient
residual interventions, 2,515 satisfy the conditions and 2,514 yield informative
bounds, again without a bound violation (\cref{app:exp39}). A separate $q=2$ bridge has 90
cases through $\epsilon=10^{-2}$ and retains all 18 invalid boundary cases.
Endpoint error is monotone for all 18 seed--adapter units; among 41 fully
uncensored cases, the post-perturbation local half-life predicts pinning time
with log-correlation $0.9996$ and median ratio $1.012$ (\cref{app:exp40}).
These structured perturbation studies test pre-rollout predictive accuracy, not arbitrary fault
statistics or a global consequence of the local proposition.

\begin{figure}[t]
\centering
\includegraphics[width=\linewidth]{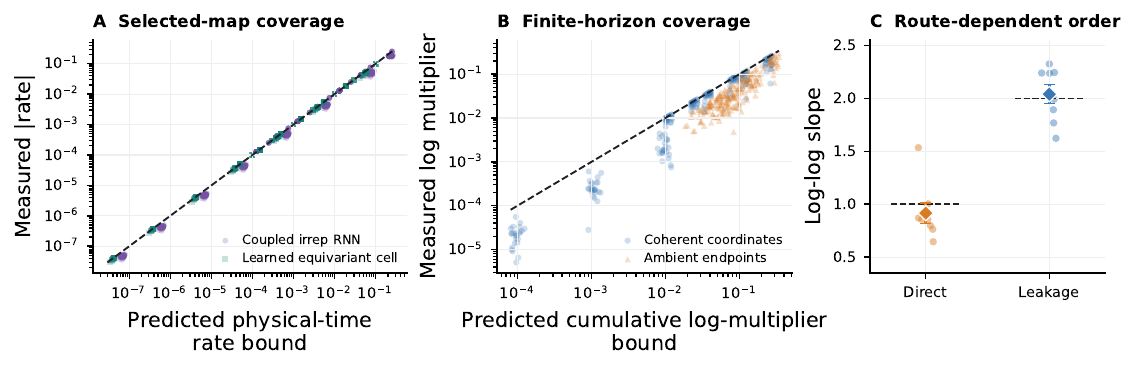}
\caption{\textbf{Local and finite-horizon perturbation bounds.}
(A) Local repeated-map cases satisfying the conditions remain below their
physical-time rate bounds. (B) All 216 finite-horizon cases satisfying the
conditions remain within the coordinate and ambient endpoint bounds. (C)
Seed-level slopes are $0.916\pm0.095$ for direct damage and $2.040\pm0.092$ for
leakage (mean $\pm$ SEM over eight base-model seeds); all 24
assumption-violation controls were identified. The finite-horizon analysis is
offline and trajectory conditioned, not an asymptotic or bi-infinite result.}
\label{fig:certificate}
\end{figure}

\begin{figure}[t]
\centering
\includegraphics[width=\linewidth]{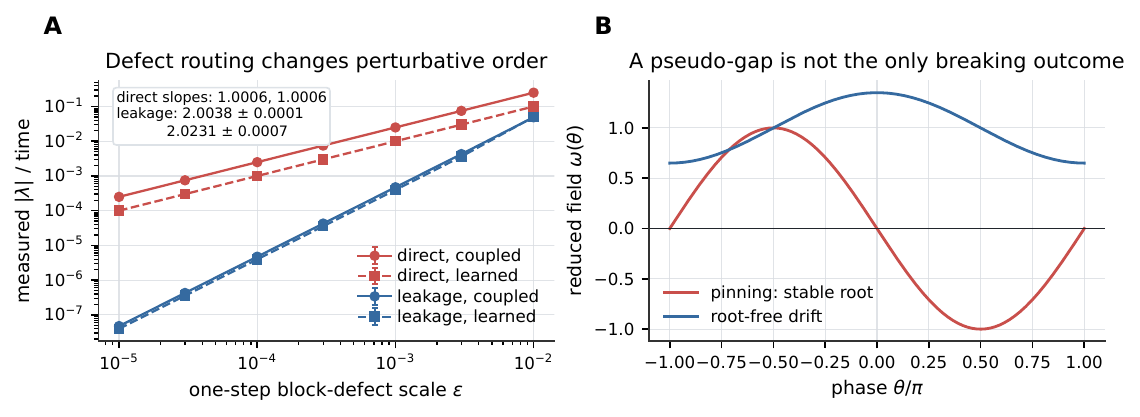}
\caption{\textbf{How defects reach continuous memory.} (A) Slopes
are first estimated within each recurrent instance and then summarized across
six model instances per family. Direct tangent defects give slope $1.0006$;
bidirectional leakage gives $2.0038\pm0.0001$ in coupled models and
$2.0231\pm0.0007$ in learned cells (mean $\pm$ SEM across models). (B) The
reduced field is $\omega(\theta)=\dot\theta/\epsilon$. A
symmetry-broken reduced phase field can pin memory at a stable root, producing
a local pseudo-gap, or remain root-free and cause systematic phase drift while
retaining zero asymptotic tangent exponent. The tiny direct-family SEM reflects
the shared analytic intervention, which largely fixes its slope; it is not
evidence of universal slope concentration. Thus a pseudo-gap is an important
breaking outcome, not a universal one.}
\label{fig:defect_routing}
\end{figure}

Figure~\ref{fig:defect_routing} distinguishes pseudo-gap opening from root-free
drift. For a reduced phase equation $\dot\theta=\epsilon\omega(\theta)$ with
$\epsilon>0$, a stable zero of
$\omega$ yields pinning and a signed local exponent
$\lambda_{\rm sym}=\epsilon\omega'(\theta_*)<0$; we call
$\gamma=-\lambda_{\rm sym}>0$ the stable pseudo-gap magnitude. If continuous
periodic $\omega$ is nowhere zero, compactness of the circle makes its sign
fixed and $|\omega|$ bounded away from zero. The phase then drifts by
$O(\epsilon t)$, while the tangent multiplier
$\omega(\theta_t)/\omega(\theta_0)$ stays bounded above and below and the
asymptotic exponent is zero. Memory can therefore fail through either
contraction/expansion or systematic drift; a Lyapunov pseudo-gap diagnoses only
the former \citep{kilpatrick2013wandering,agmon2023multiple}.

\section{Noisy continuous memory: local precision}
\label{sec:stochastic}
Exact symmetry prevents deterministic exponential forgetting along the group
orbit, but it does not prevent noise-driven wandering. The local precision of
the decoded variable depends jointly on the orbit embedding, ambient noise, and
decoder \citep{burakFiete2012,kilpatrick2013wandering}.

The full action $A_x:\g\to T_xM$ need not be injective when a stabilizer is
nontrivial. Fix a local quotient-coordinate metric and an isometric horizontal
frame $\iota_x:\R^q\to(\ker A_x)^\perp\subset\g$, and define the reduced
full-column-rank action
\begin{equation}
 \bar A_x:=A_x\iota_x:\R^q\longrightarrow E_x^G.
 \label{eq:reduced_action}
\end{equation}
Its singular values, condition number, and energy are relative to that chosen
quotient-coordinate metric. In the free torus experiments, the
natural Lie-algebra coordinates already give $\bar A_x=A_x$.

Let $\psi$ be a local $q$-dimensional memory coordinate and set
$Z_x=D\psi_x$. First-order faithfulness to the selected quotient coordinates
means that $Z_x$ is a left inverse of $\bar A_x$:
\begin{equation}
 Z_x\bar A_x=I_q.
 \label{eq:left_inverse}
\end{equation}
For an ambient It\^o increment with covariance $Q_x\,\dd t$, It\^o's formula
\citep{oksendal2003} gives the local quadratic variation in memory coordinates,
\begin{equation}
 D_\psi(x)=Z_xQ_xZ_x^\top.
 \label{eq:decoder_covariance}
\end{equation}
This is the correct general statement: the action matrix alone does not fix
coordinate diffusion unless the decoder is specified. Throughout, ``diffusion
tensor'' means the quadratic-variation (covariance-rate) matrix in
\cref{eq:decoder_covariance}; a Fokker--Planck convention would place a factor of
$1/2$ in front of this matrix. Curvature can add an It\^o drift, and long-time
phase diffusion also depends on the surrounding dynamics and isochrons
\citep{maclaurin2023phase,ahsan2025covariance}.

\begin{corollary}[One-step discrete covariance]
\label{cor:discrete_covariance}
Let $X_{k+1}=F(X_k)+\eta_{\Delta t}$, set $y=F(x)$, and condition on
$X_k=x$. If $\E\eta_{\Delta t}=0$,
$\operatorname{Cov}(\eta_{\Delta t})=Q_y\Delta t$,
$\psi$ is differentiable at $y$, and its Taylor remainder satisfies
\[
 \E\norm{\psi(y+\eta_{\Delta t})-\psi(y)-Z_y\eta_{\Delta t}}^2
 =O(\Delta t^2),
\]
then
\begin{equation}
 \operatorname{Cov}\!\left[\psi(X_{k+1})\mid X_k=x\right]
 =Z_yQ_yZ_y^\top\Delta t+o(\Delta t).
 \label{eq:discrete_covariance}
\end{equation}
\end{corollary}

\begin{theorem}[Decoder-aware isotropic noise floor]
\label{thm:decoder_floor}
Assume $Q_x=\sigma^2 I$ with $\sigma>0$. Every local decoder satisfying
\cref{eq:left_inverse} obeys
\begin{equation}
 D_\psi(x)\succeq \sigma^2(\bar A_x^\top \bar A_x)^{-1}.
 \label{eq:psd_floor}
\end{equation}
Equality holds if and only if $Z_x=\bar A_x^\dagger$, equivalently the decoder
has no first-order sensitivity in directions orthogonal to
$\range(\bar A_x)=E_x^G$.
\end{theorem}

Indeed, $Z=\bar A^\dagger+R$ with $R\bar A=0$, so
$ZZ^\top=(\bar A^\top\bar A)^{-1}+RR^\top$. More generally, for $Q\succ0$ the
minimum-covariance left inverse is the generalized least-squares decoder
\begin{equation}
 Z_*= (\bar A^\top Q^{-1}\bar A)^{-1}\bar A^\top Q^{-1},
 \qquad D_{\psi,*}=(\bar A^\top Q^{-1}\bar A)^{-1}.
 \label{eq:gls_decoder}
\end{equation}
This is the classical Gauss--Markov/Aitken generalized least-squares solution
\citep{aitken1936}. Our contribution is its decoder-coordinate interpretation
for recurrent memory and the causal fixed-budget test below, not the linear
estimation formula itself.

\begin{corollary}[Scaled-isometric action optimum]
\label{cor:tight}
Under the assumptions of \cref{thm:decoder_floor}, fix tangent energy
$P=\tr(\bar A_x^\top\bar A_x)$. Then every faithful decoder satisfies
\begin{equation}
 \lambda_{\max}(D_\psi)\ge\frac{\sigma^2q}{P},
 \qquad
 \tr(D_\psi)\ge\frac{\sigma^2q^2}{P}.
 \label{eq:tight_bounds}
\end{equation}
Both lower bounds are attained exactly when
$\bar A_x^\top\bar A_x=(P/q)I_q$ and $Z_x=\bar A_x^\dagger$.
\end{corollary}

Thus the representation optimum is conditional on isotropic ambient noise, a
chosen group-coordinate metric, a fixed tangent-energy budget, and a
normal-insensitive decoder. The algebra is the familiar tight-frame/scaled-
isometry condition \citep{benedettoFickus2003,xu2023lie,xu2025grid}; the
contribution here is its recurrent-memory decoder interpretation and causal
validation, not the frame inequality itself.

\paragraph{Paired $T^2$ geometry intervention.}
We fix $P=2$ and prescribe $\kappa(\bar A)\in\{1,1.5,2,3,5,8\}$ in
scale-conjugate $T^2$ integrators with matched noiseless behavior. Six paired
seed units, not 36 independent models, isolate the response to isotropic
ambient noise. The angle decoder has differential $\bar A^\dagger$ on the
target orbit. Median tensor error is $0.56\%$ locally and $3.32\%$ over long
zero-input rollouts. Increasing $\kappa$ from 1 to 8 raises noisy horizon-256
worst-coordinate RMSE by $0.639\pm0.011$ rad (mean $\pm$ SEM), with all six
pairs increasing. Gaussian increments follow updates at $\Delta t=0.1$:
\cref{cor:discrete_covariance} is a small-step bridge, not an exact identity
for nonlinear finite-horizon task error. Full controls appear in the appendix.

\paragraph{Theory-guided learned geometry.}
We next replace the prescribed radii by the exact fixed-budget
parameterization $r_i^2=2\,\mathrm{softmax}(\alpha)_i$ and initialize all
learnable arms at $\kappa=8$ (\cref{app:exp41}). A clean-only pilot selected
the covariance-regularization coefficient without using noisy test error.
With that coefficient and the training settings held fixed, our principal
$T^2$ learning evaluation uses twenty fresh paired training units with random
streams separated across roles and units (\cref{app:fresh_replication}).
Hyperparameter selection and the separate historical cohort are documented
in \cref{app:geometry_selection}. Those historical observations are descriptive
and are not pooled with the replication.
The covariance objective lowers noisy horizon-256 RMSE in $20/20$ pairs:
the mean change is $-0.623$ rad (95\% paired $t$ interval
$[-0.630,-0.616]$), while the clean-change interval lies within the
prespecified $\pm0.01$ rad margin. Its noisy RMSE remains
$0.050\pm0.001$ above the isotropic oracle (mean $\pm$ SEM).
Direct noise training also improves noisy memory in $20/20$ pairs, but its
clean change is $+0.018$ rad (95\% interval $[0.006,0.031]$), failing that
equivalence criterion. The benefit is therefore objective-specific, not a
blanket claim of cost-free robustness.

\paragraph{Noise-dependent action geometry.}
For coupled $T^q$ integrators with fixed angle decoding, block noise
$Q_i=\sigma^2v_iI_2$, and Euclidean budget $\sum_i r_i^2=P$,
$D_\psi=\diag(\sigma^2v_i/r_i^2)$ has trace-minimizing allocation
$r_i^2=P\sqrt{v_i}/\sum_j\sqrt{v_j}$ (\cref{app:noise_geometry}).
Sixteen paired training units per dimension test this prediction at $P=q$.
At horizon 256 (physical time 25.6), speed 1.5, anisotropic-covariance minus
isotropic regularization changes mean endpoint phase MSE by $-0.006106$
rad$^2$ for $T^4$ (primary; 95\% paired $t$ interval
$[-0.006674,-0.005538]$) and $-0.006188$ for $T^8$ (secondary).
Switching the same trained cells to isotropic evaluation noise reverses these
differences to $+0.014229$ and $+0.009814$ rad$^2$.
This tests noise-specific precision, not a universal isometric optimum or
the deterministic contraction assumptions.

\begin{figure}[t]
\centering
\includegraphics[width=\linewidth]{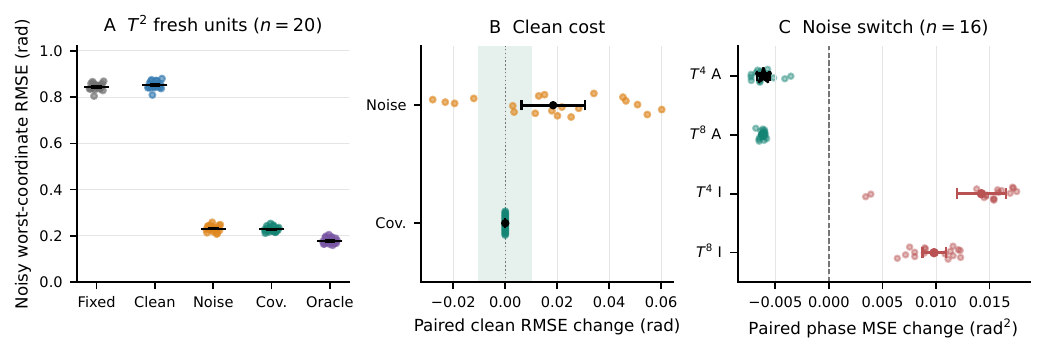}
\caption{\textbf{Noise-specific geometry and clean-error cost.} (A) $T^2$ noisy
worst-coordinate RMSE at horizon 256; dots are twenty paired training units,
black marks mean $\pm$ SEM. (B) Clean RMSE changes relative to clean learning;
bars are paired 95\% $t$ intervals and shading is the $\pm0.01$ rad
equivalence margin; covariance but not noise training meets it.
(C) Anisotropic-covariance minus isotropic regularization: endpoint phase
MSE changes under anisotropic (A) and isotropic (I) noise, with sixteen
paired units and 95\% $t$ intervals per row; the star marks the primary
$T^4$ contrast. The subtraction is unchanged across noise conditions.
Panels A/B use rad, C rad$^2$. These are conditional task consequences, not
theorem proof or a claim that one geometry always wins.}
\label{fig:diffusion}
\end{figure}

\section{Related work}
\label{sec:related}
Prior work covers continuous attractors, equivariant networks, sequence/world
models, and robust memory
\citep{seung1996,seung1998,burakFiete2009,cohenWelling2016,
bronstein2021geometric,zhang2022translation,zuo2023temporal,diBernardo2025,
keller2025fernn,lillemark2026flow,can2025robust}. We specialize classical
equivariant, persistence, nonautonomous, and subspace-perturbation theory
\citep{golubitsky1988,krupa1990,fenichel1971,hirschPughShub1977,
sackerSell1976iii,coppel1978,potzsche2010,wedin1972} to block-resolved recurrent
memory, not a general dichotomy theorem. Adjacent work treats symmetry
relaxation, Goldstone-like propagation, certified rollouts, Lie-group
navigation, generalized least squares, and tight frames
\citep{pertigkiozoglou2026recm,iqbal2026goldstone,wang2026certified,
wang2026conformal,xu2023lie,xu2025grid,aitken1936,benedettoFickus2003}.
We combine stabilizer, block-defect, and decoder-covariance analysis. Relative
to \citet{mo2026symmetry}, this work adds controlled transport, finite-horizon
bounds, and decoder-aware learning; \cref{app:related_extended} gives detailed
comparisons.

\section{Limitations and conclusion}
\label{sec:limitations}
Protected rank $q$ requires per-input equivariance and a nondegenerate orbit
stratum; approximate equivariance requires separate bounds. The time-varying
perturbation bound is finite-horizon and offline, assumes normal contraction,
and makes no asymptotic, online, or bi-infinite claim. The stochastic optimum depends on the
declared noise, metric, decoder, and energy budget. The $T^4/T^8$ checks use a
fixed angle decoder and block noise; they establish neither global stochastic
stability nor deterministic contraction. Within these limits, orbit
geometry, defect routing, and decoder geometry govern memory dimension,
deterministic degradation, and diffusion.

\label{end:main_text}
\subsection*{Reproducibility statement}
Complete proofs and experimental protocols are provided in the appendix. The
supplementary material contains the Lean development, theorem-to-declaration
mapping, Python implementation, study configurations, checkpoints, result
records, and figure-generation or correction scripts. Numerical precision,
random-stream separation, and retained unsuccessful trials are documented
with the corresponding studies; \cref{app:reproducibility} summarizes the
computational resources and verification scope.

\subsection*{AI use statement}
\label{sec:ai_use}
Large language models assisted with critiquing mathematical claims, revising proofs,
and checking the Lean formalization. They also assisted with literature discovery,
figure preparation, and manuscript editing. Reported measurements
come from the archived computations, not from generated numerical claims.
The formal and executable checks, including their limitations, are documented
in \cref{app:formal_correspondence} and the supplementary artifact; they do not
establish end-to-end formal verification of the manuscript or experimental
pipeline. 

\clearpage
\bibliographystyle{iclr2027_conference}
\bibliography{references}

@article{seung1996,
  author = {Seung, H. S.},
  title = {How the Brain Keeps the Eyes Still},
  journal = {Proceedings of the National Academy of Sciences of the United States of America},
  volume = {93},
  number = {23},
  pages = {13339--13344},
  year = {1996},
  doi = {10.1073/pnas.93.23.13339}
}

@article{seung1998,
  author = {Seung, H. S.},
  title = {Continuous Attractors and Oculomotor Control},
  journal = {Neural Networks},
  volume = {11},
  number = {7--8},
  pages = {1253--1258},
  year = {1998},
  doi = {10.1016/S0893-6080(98)00064-1}
}

@article{burakFiete2009,
  author = {Burak, Yoram and Fiete, Ila R.},
  title = {Accurate Path Integration in Continuous Attractor Network Models of Grid Cells},
  journal = {PLOS Computational Biology},
  volume = {5},
  number = {2},
  pages = {e1000291},
  year = {2009},
  doi = {10.1371/journal.pcbi.1000291}
}

@inproceedings{cuevaWei2018,
  author = {Cueva, Christopher J. and Wei, Xue-Xin},
  title = {Emergence of Grid-like Representations by Training Recurrent Neural Networks to Perform Spatial Localization},
  booktitle = {International Conference on Learning Representations},
  year = {2018},
  url = {https://openreview.net/forum?id=B17JTOe0-},
  eprint = {1803.07770},
  archivePrefix = {arXiv},
  primaryClass = {q-bio.NC}
}

@article{banino2018,
  author = {Banino, Andrea and Barry, Caswell and Uria, Benigno and Blundell, Charles and Lillicrap, Timothy and Mirowski, Piotr and Pritzel, Alexander and Chadwick, Martin J. and Degris, Thomas and Modayil, Joseph and Wayne, Greg and Soyer, Hubert and Viola, Fabio and Zhang, Brian and Goroshin, Ross and Rabinowitz, Neil and Pascanu, Razvan and Beattie, Charles and Petersen, Stig and Sadik, Amir and Gaffney, Stephen and King, Helen and Kavukcuoglu, Koray and Hassabis, Demis and Hadsell, Raia and Kumaran, Dharshan},
  title = {Vector-Based Navigation Using Grid-like Representations in Artificial Agents},
  journal = {Nature},
  volume = {557},
  number = {7705},
  pages = {429--433},
  year = {2018},
  doi = {10.1038/s41586-018-0102-6}
}

@article{sorscher2023grid,
  author = {Sorscher, Ben and Mel, Gabriel C. and Ocko, Samuel A. and Giocomo, Lisa M. and Ganguli, Surya},
  title = {A Unified Theory for the Computational and Mechanistic Origins of Grid Cells},
  journal = {Neuron},
  volume = {111},
  number = {1},
  pages = {121--137.e13},
  year = {2023},
  doi = {10.1016/j.neuron.2022.10.003}
}

@inproceedings{xu2023lie,
  author = {Xu, Dehong and Gao, Ruiqi and Zhang, Wen-Hao and Wei, Xue-Xin and Wu, Ying Nian},
  title = {Conformal Isometry of {Lie} Group Representation in Recurrent Network of Grid Cells},
  booktitle = {Proceedings of the 1st NeurIPS Workshop on Symmetry and Geometry in Neural Representations},
  series = {Proceedings of Machine Learning Research},
  volume = {197},
  pages = {370--387},
  publisher = {PMLR},
  year = {2023},
  url = {https://proceedings.mlr.press/v197/xu23a.html}
}

@inproceedings{xu2025grid,
  author = {Xu, Dehong and Gao, Ruiqi and Zhang, Wenhao and Wei, Xue-Xin and Wu, Ying Nian},
  title = {On Conformal Isometry of Grid Cells: Learning Distance-Preserving Position Embedding},
  booktitle = {International Conference on Learning Representations},
  year = {2025},
  note = {Oral presentation},
  url = {https://openreview.net/forum?id=Xo0Q1N7CGk}
}

@inproceedings{wang2022approximately,
  title = {Approximately Equivariant Networks for Imperfectly Symmetric Dynamics},
  author = {Wang, Rui and Walters, Robin and Yu, Rose},
  booktitle = {Proceedings of the 39th International Conference on Machine Learning},
  series = {Proceedings of Machine Learning Research},
  volume = {162},
  pages = {23078--23091},
  year = {2022},
  publisher = {PMLR},
  url = {https://proceedings.mlr.press/v162/wang22aa.html}
}

@article{diBernardo2025,
  author = {Di Bernardo, Arianna and Valente, Adrian and Mastrogiuseppe, Francesca and Ostojic, Srdjan},
  title = {Shaping Manifolds in Equivariant Recurrent Neural Networks},
  journal = {arXiv preprint arXiv:2511.04802},
  year = {2025},
  doi = {10.48550/arXiv.2511.04802},
  url = {https://arxiv.org/abs/2511.04802},
  eprint = {2511.04802},
  archivePrefix = {arXiv},
  primaryClass = {q-bio.NC}
}

@article{benettin1980Theory,
  author = {Benettin, G. and Galgani, L. and Giorgilli, A. and Strelcyn, J.-M.},
  title = {Lyapunov Characteristic Exponents for Smooth Dynamical Systems and for Hamiltonian Systems; A Method for Computing All of Them. Part 1: Theory},
  journal = {Meccanica},
  volume = {15},
  number = {1},
  pages = {9--20},
  year = {1980},
  doi = {10.1007/BF02128236}
}

@article{benettin1980Numerical,
  author = {Benettin, G. and Galgani, L. and Giorgilli, A. and Strelcyn, J.-M.},
  title = {Lyapunov Characteristic Exponents for Smooth Dynamical Systems and for Hamiltonian Systems; A Method for Computing All of Them. Part 2: Numerical Application},
  journal = {Meccanica},
  volume = {15},
  number = {1},
  pages = {21--30},
  year = {1980},
  doi = {10.1007/BF02128237}
}

@article{ginelli2007,
  author = {Ginelli, F. and Poggi, P. and Turchi, A. and Chat{\'e}, H. and Livi, R. and Politi, A.},
  title = {Characterizing Dynamics with Covariant Lyapunov Vectors},
  journal = {Physical Review Letters},
  volume = {99},
  number = {13},
  pages = {130601},
  year = {2007},
  doi = {10.1103/PhysRevLett.99.130601}
}

@book{golubitsky1988,
  author = {Golubitsky, Martin and Stewart, Ian and Schaeffer, David G.},
  title = {Singularities and Groups in Bifurcation Theory, Volume II},
  series = {Applied Mathematical Sciences},
  volume = {69},
  publisher = {Springer},
  year = {1988},
  doi = {10.1007/978-1-4612-4574-2}
}

@article{krupa1990,
  author = {Krupa, Martin},
  title = {Bifurcations of Relative Equilibria},
  journal = {SIAM Journal on Mathematical Analysis},
  volume = {21},
  number = {6},
  pages = {1453--1486},
  year = {1990},
  doi = {10.1137/0521081}
}

@article{knyazevArgentati2002,
  author = {Knyazev, Andrew V. and Argentati, Merico E.},
  title = {Principal Angles between Subspaces in an {A}-Based Scalar Product: Algorithms and Perturbation Estimates},
  journal = {SIAM Journal on Scientific Computing},
  volume = {23},
  number = {6},
  pages = {2008--2040},
  year = {2002},
  doi = {10.1137/S1064827500377332}
}

@article{oseledets1968,
  author = {Oseledets, V. I.},
  title = {A Multiplicative Ergodic Theorem: Lyapunov Characteristic Numbers for Dynamical Systems},
  journal = {Transactions of the Moscow Mathematical Society},
  volume = {19},
  pages = {197--231},
  year = {1968}
}

@article{froyland2013semi,
  author = {Froyland, Gary and Lloyd, Simon and Quas, Anthony},
  title = {A Semi-Invertible {Oseledets} Theorem with Applications to Transfer Operator Cocycles},
  journal = {Discrete and Continuous Dynamical Systems},
  volume = {33},
  number = {9},
  pages = {3835--3860},
  year = {2013},
  doi = {10.3934/dcds.2013.33.3835}
}

@article{aitken1936,
  author = {Aitken, A. C.},
  title = {On Least Squares and Linear Combination of Observations},
  journal = {Proceedings of the Royal Society of Edinburgh},
  volume = {55},
  pages = {42--48},
  year = {1936},
  doi = {10.1017/S0370164600014346}
}

@inproceedings{engelken2023gradientflossing,
  author = {Engelken, Rainer},
  title = {Gradient Flossing: Improving Gradient Descent through Dynamic Control of Jacobians},
  booktitle = {Advances in Neural Information Processing Systems},
  volume = {36},
  year = {2023},
  url = {https://proceedings.neurips.cc/paper_files/paper/2023/hash/214ce905bf2072535e34b3cf873cbbc8-Abstract-Conference.html},
  eprint = {2312.17306},
  archivePrefix = {arXiv},
  primaryClass = {cs.LG}
}

@inproceedings{cho2014gru,
  author = {Cho, Kyunghyun and van Merri{\"e}nboer, Bart and Gulcehre, Caglar and Bahdanau, Dzmitry and Bougares, Fethi and Schwenk, Holger and Bengio, Yoshua},
  title = {Learning Phrase Representations Using {RNN} Encoder--Decoder for Statistical Machine Translation},
  booktitle = {Proceedings of the 2014 Conference on Empirical Methods in Natural Language Processing},
  pages = {1724--1734},
  year = {2014},
  doi = {10.3115/v1/D14-1179}
}

@inproceedings{cohenWelling2016,
  author = {Cohen, Taco and Welling, Max},
  title = {Group Equivariant Convolutional Networks},
  booktitle = {Proceedings of the 33rd International Conference on Machine Learning},
  series = {Proceedings of Machine Learning Research},
  volume = {48},
  pages = {2990--2999},
  year = {2016},
  url = {https://proceedings.mlr.press/v48/cohenc16.html}
}

@inproceedings{kondorTrivedi2018,
  author = {Kondor, Risi and Trivedi, Shubhendu},
  title = {On the Generalization of Equivariance and Convolution in Neural Networks to the Action of Compact Groups},
  booktitle = {Proceedings of the 35th International Conference on Machine Learning},
  series = {Proceedings of Machine Learning Research},
  volume = {80},
  pages = {2747--2755},
  year = {2018},
  url = {https://proceedings.mlr.press/v80/kondor18a.html}
}

@article{bronstein2021geometric,
  author = {Bronstein, Michael M. and Bruna, Joan and Cohen, Taco and Veli{\v{c}}kovi{\'c}, Petar},
  title = {Geometric Deep Learning: Grids, Groups, Graphs, Geodesics, and Gauges},
  journal = {arXiv preprint arXiv:2104.13478},
  year = {2021},
  doi = {10.48550/arXiv.2104.13478},
  url = {https://arxiv.org/abs/2104.13478},
  eprint = {2104.13478},
  archivePrefix = {arXiv},
  primaryClass = {cs.LG}
}

@article{hochreiter1997,
  author = {Hochreiter, Sepp and Schmidhuber, J{\"u}rgen},
  title = {Long Short-Term Memory},
  journal = {Neural Computation},
  volume = {9},
  number = {8},
  pages = {1735--1780},
  year = {1997},
  doi = {10.1162/neco.1997.9.8.1735}
}

@inproceedings{henaff2016,
  author = {Henaff, Mikael and Szlam, Arthur and LeCun, Yann},
  title = {Recurrent Orthogonal Networks and Long-Memory Tasks},
  booktitle = {Proceedings of the 33rd International Conference on Machine Learning},
  series = {Proceedings of Machine Learning Research},
  volume = {48},
  pages = {2034--2042},
  year = {2016},
  url = {https://proceedings.mlr.press/v48/henaff16.html}
}

@inproceedings{arjovsky2016,
  author = {Arjovsky, Martin and Shah, Amar and Bengio, Yoshua},
  title = {Unitary Evolution Recurrent Neural Networks},
  booktitle = {Proceedings of the 33rd International Conference on Machine Learning},
  series = {Proceedings of Machine Learning Research},
  volume = {48},
  pages = {1120--1128},
  year = {2016},
  url = {https://proceedings.mlr.press/v48/arjovsky16.html}
}

@article{rumberger2001,
  author = {Rumberger, Matthias},
  title = {Lyapunov Exponents on the Orbit Space},
  journal = {Discrete and Continuous Dynamical Systems},
  volume = {7},
  number = {1},
  pages = {91--113},
  year = {2001},
  doi = {10.3934/dcds.2001.7.91},
  url = {https://www.aimsciences.org/article/doi/10.3934/dcds.2001.7.91}
}

@inproceedings{sagodi2024back,
  author = {S{\'a}godi, {\'A}bel and Mart{\'i}n-S{\'a}nchez, Guillermo and Sok{\'o}{\l}, Piotr and Park, Il Memming},
  title = {Back to the Continuous Attractor},
  booktitle = {Advances in Neural Information Processing Systems},
  volume = {37},
  year = {2024},
  url = {https://proceedings.neurips.cc/paper_files/paper/2024/hash/7b78a2a7360d5a9ad750834dc5a33bfb-Abstract-Conference.html}
}

@article{darshanRivkind2022,
  author = {Darshan, Ran and Rivkind, Alexander},
  title = {Learning to Represent Continuous Variables in Heterogeneous Neural Networks},
  journal = {Cell Reports},
  volume = {39},
  number = {1},
  pages = {110612},
  year = {2022},
  doi = {10.1016/j.celrep.2022.110612}
}

@inproceedings{keller2025fernn,
  author = {Keller, T. Anderson},
  title = {Flow Equivariant Recurrent Neural Networks},
  booktitle = {Advances in Neural Information Processing Systems},
  volume = {38},
  year = {2025},
  note = {Spotlight},
  doi = {10.48550/arXiv.2507.14793},
  url = {https://proceedings.neurips.cc/paper_files/paper/2025/hash/e637029c42aa593850eeebf46616444d-Abstract-Conference.html},
  eprint = {2507.14793},
  archivePrefix = {arXiv},
  primaryClass = {cs.LG}
}

@inproceedings{lillemark2026flow,
  author = {Lillemark, Hansen Jin and Huang, Benhao and Zhan, Fangneng and Du, Yilun and Keller, T. Anderson},
  title = {Flow Equivariant World Models: Memory for Partially Observed Dynamic Environments},
  booktitle = {International Conference on Machine Learning},
  year = {2026},
  note = {Accepted},
  doi = {10.48550/arXiv.2601.01075},
  url = {https://arxiv.org/abs/2601.01075},
  eprint = {2601.01075},
  archivePrefix = {arXiv},
  primaryClass = {cs.LG}
}

@article{wang2026certified,
  author = {Wang, Hongbo},
  title = {Certified World Models: Predictability Across Configuration, Horizon, and Resolution},
  journal = {arXiv preprint arXiv:2606.13092},
  year = {2026},
  doi = {10.48550/arXiv.2606.13092},
  url = {https://arxiv.org/abs/2606.13092},
  eprint = {2606.13092},
  archivePrefix = {arXiv},
  primaryClass = {cs.LG}
}

@article{pertigkiozoglou2026recm,
  author = {Pertigkiozoglou, Stefanos and Petrache, Mircea and Trivedi, Shubhendu and Daniilidis, Kostas},
  title = {Recurrent Equivariant Constraint Modulation: Learning Per-Layer Symmetry Relaxation from Data},
  journal = {arXiv preprint arXiv:2602.02853},
  year = {2026},
  doi = {10.48550/arXiv.2602.02853},
  url = {https://arxiv.org/abs/2602.02853},
  eprint = {2602.02853},
  archivePrefix = {arXiv},
  primaryClass = {cs.LG}
}

@article{wang2026conformal,
  author = {Wang, Hongbo},
  title = {Conformal Orbit-Valid Trust Horizons for Equivariant World Models},
  journal = {arXiv preprint arXiv:2606.24946},
  year = {2026},
  doi = {10.48550/arXiv.2606.24946},
  url = {https://arxiv.org/abs/2606.24946},
  eprint = {2606.24946},
  archivePrefix = {arXiv},
  primaryClass = {cs.LG}
}

@inproceedings{liang2025symreg,
  author = {Liang, Arthur and S{\'a}godi, {\'A}bel and Sok{\'o}{\l}, Piotr A. and Park, Il Memming},
  title = {Symmetry-Regularized Learning of Continuous Attractor Dynamics},
  booktitle = {NeurIPS 2025 Workshop on Symmetry and Geometry in Neural Representations (NeurReps)},
  year = {2025},
  note = {Poster},
  url = {https://openreview.net/forum?id=W8Gf7CYCo8}
}

@article{burakFiete2012,
  author = {Burak, Yoram and Fiete, Ila R.},
  title = {Fundamental Limits on Persistent Activity in Networks of Noisy Neurons},
  journal = {Proceedings of the National Academy of Sciences of the United States of America},
  volume = {109},
  number = {43},
  pages = {17645--17650},
  year = {2012},
  doi = {10.1073/pnas.1117386109}
}

@article{fenichel1971,
  author = {Fenichel, Neil},
  title = {Persistence and Smoothness of Invariant Manifolds for Flows},
  journal = {Indiana University Mathematics Journal},
  volume = {21},
  number = {3},
  pages = {193--226},
  year = {1971},
  doi = {10.1512/iumj.1972.21.21017}
}

@book{hirschPughShub1977,
  author = {Hirsch, Morris W. and Pugh, Charles C. and Shub, Michael},
  title = {Invariant Manifolds},
  series = {Lecture Notes in Mathematics},
  volume = {583},
  publisher = {Springer-Verlag},
  address = {Berlin},
  year = {1977},
  doi = {10.1007/BFb0092042}
}

@book{oksendal2003,
  author = {{\O}ksendal, Bernt},
  title = {Stochastic Differential Equations: An Introduction with Applications},
  edition = {6},
  series = {Universitext},
  publisher = {Springer Berlin, Heidelberg},
  year = {2003},
  doi = {10.1007/978-3-642-14394-6}
}

@article{maclaurin2023phase,
  author = {MacLaurin, J.},
  title = {Phase Reduction of Waves, Patterns, and Oscillations Subject to Spatially Extended Noise},
  journal = {SIAM Journal on Applied Mathematics},
  volume = {83},
  number = {3},
  pages = {1215--1244},
  year = {2023},
  doi = {10.1137/21M1451221}
}

@article{ahsan2025covariance,
  author = {Ahsan, Zaid and Dankowicz, Harry and Kuehn, Christian},
  title = {Adjoint-Based Projections for Quantifying Statistical Covariance near Stochastically Perturbed Limit Cycles and Tori},
  journal = {SIAM Journal on Applied Dynamical Systems},
  volume = {24},
  number = {2},
  pages = {1455--1493},
  year = {2025},
  doi = {10.1137/23M1599318}
}

@article{benedettoFickus2003,
  author = {Benedetto, John J. and Fickus, Matthew},
  title = {Finite Normalized Tight Frames},
  journal = {Advances in Computational Mathematics},
  volume = {18},
  pages = {357--385},
  year = {2003},
  doi = {10.1023/A:1021323312367}
}

@article{iqbal2026goldstone,
  author = {Iqbal, Nabil and Keller, T. Anderson and Song, Yue and Miyato, Takeru and Welling, Max},
  title = {Spontaneous Symmetry Breaking and Goldstone Modes for Deep Information Propagation},
  journal = {arXiv preprint arXiv:2605.14685},
  year = {2026},
  doi = {10.48550/arXiv.2605.14685},
  url = {https://arxiv.org/abs/2605.14685},
  eprint = {2605.14685},
  archivePrefix = {arXiv},
  primaryClass = {cs.LG}
}

@article{mo2026symmetry,
  author = {Mo, Hanson Hanxuan},
  title = {Symmetry-Protected Lyapunov Neutral Modes in Equivariant Recurrent Networks},
  journal = {arXiv preprint arXiv:2605.03338},
  year = {2026},
  doi = {10.48550/arXiv.2605.03338},
  url = {https://arxiv.org/abs/2605.03338},
  eprint = {2605.03338},
  archivePrefix = {arXiv},
  primaryClass = {cs.LG}
}

@article{nartallo2026driftdiffusion,
  author = {Nartallo-Kaluarachchi, Ram{\'o}n and Lambiotte, Renaud and Goriely, Alain},
  title = {Drift--Diffusion Matching: Embedding Dynamics in Latent Manifolds of Asymmetric Neural Networks},
  journal = {arXiv preprint arXiv:2602.14885},
  year = {2026},
  doi = {10.48550/arXiv.2602.14885},
  url = {https://arxiv.org/abs/2602.14885},
  eprint = {2602.14885},
  archivePrefix = {arXiv},
  primaryClass = {cond-mat.dis-nn}
}

@article{sagodiPark2026uap,
  author = {S{\'a}godi, {\'A}bel and Park, Il Memming},
  title = {Universal Approximation Theorems for Dynamical Systems with Infinite-Time Horizon Guarantees},
  journal = {arXiv preprint arXiv:2602.08640},
  year = {2026},
  doi = {10.48550/arXiv.2602.08640},
  url = {https://arxiv.org/abs/2602.08640},
  eprint = {2602.08640},
  archivePrefix = {arXiv},
  primaryClass = {math.DS}
}

@article{sackerSell1976iii,
  author = {Sacker, Robert J. and Sell, George R.},
  title = {Existence of Dichotomies and Invariant Splittings for Linear Differential Systems, III},
  journal = {Journal of Differential Equations},
  volume = {22},
  number = {2},
  pages = {497--522},
  year = {1976},
  doi = {10.1016/0022-0396(76)90043-7}
}

@book{coppel1978,
  author = {Coppel, W. A.},
  title = {Dichotomies in Stability Theory},
  series = {Lecture Notes in Mathematics},
  volume = {629},
  publisher = {Springer-Verlag},
  address = {Berlin},
  year = {1978},
  doi = {10.1007/BFb0067780}
}

@book{potzsche2010,
  author = {P{\"o}tzsche, Christian},
  title = {Geometric Theory of Discrete Nonautonomous Dynamical Systems},
  series = {Lecture Notes in Mathematics},
  volume = {2002},
  publisher = {Springer},
  address = {Berlin, Heidelberg},
  year = {2010},
  doi = {10.1007/978-3-642-14258-1}
}

@article{wedin1972,
  author = {Wedin, Per-{\AA}ke},
  title = {Perturbation Bounds in Connection with Singular Value Decomposition},
  journal = {BIT Numerical Mathematics},
  volume = {12},
  pages = {99--111},
  year = {1972},
  doi = {10.1007/BF01932678}
}

@article{chaudhuri2016memory,
  author = {Chaudhuri, Rishidev and Fiete, Ila},
  title = {Computational Principles of Memory},
  journal = {Nature Neuroscience},
  volume = {19},
  number = {3},
  pages = {394--403},
  year = {2016},
  doi = {10.1038/nn.4237}
}

@article{kilpatrick2013wandering,
  author = {Kilpatrick, Zachary P. and Ermentrout, Bard},
  title = {Wandering Bumps in Stochastic Neural Fields},
  journal = {SIAM Journal on Applied Dynamical Systems},
  volume = {12},
  number = {1},
  pages = {61--94},
  year = {2013},
  doi = {10.1137/120877106}
}

@inproceedings{agmon2023multiple,
  author = {Agmon, Haggai and Burak, Yoram},
  title = {Simultaneous Embedding of Multiple Attractor Manifolds in a Recurrent Neural Network Using Constrained Gradient Optimization},
  booktitle = {Advances in Neural Information Processing Systems},
  volume = {36},
  year = {2023},
  url = {https://proceedings.neurips.cc/paper_files/paper/2023/hash/361e5112d2eca09513bbd266e4b2d2be-Abstract-Conference.html}
}

@article{palais1961slices,
  author = {Palais, Richard S.},
  title = {On the Existence of Slices for Actions of Non-Compact Lie Groups},
  journal = {Annals of Mathematics},
  volume = {73},
  number = {2},
  pages = {295--323},
  year = {1961},
  doi = {10.2307/1970335}
}

@inproceedings{zhang2022translation,
  author = {Zhang, Wenhao and Wu, Ying Nian and Wu, Si},
  title = {Translation-equivariant Representation in Recurrent Networks with a Continuous Manifold of Attractors},
  booktitle = {Advances in Neural Information Processing Systems},
  volume = {35},
  year = {2022},
  url = {https://proceedings.neurips.cc/paper_files/paper/2022/hash/65384a01325fecbd364c835db872443c-Abstract-Conference.html}
}

@article{can2025robust,
  author = {Can, Tankut and Krishnamurthy, Kamesh},
  title = {Emergence of Robust Memory Manifolds},
  journal = {PRX Life},
  volume = {3},
  number = {2},
  pages = {023006},
  year = {2025},
  doi = {10.1103/PRXLife.3.023006}
}

@article{davisKahan1970,
  author = {Davis, Chandler and Kahan, W. M.},
  title = {The Rotation of Eigenvectors by a Perturbation. III},
  journal = {SIAM Journal on Numerical Analysis},
  volume = {7},
  number = {1},
  pages = {1--46},
  year = {1970},
  doi = {10.1137/0707001}
}

@article{stewart1973subspaces,
  author = {Stewart, G. W.},
  title = {Error and Perturbation Bounds for Subspaces Associated with Certain Eigenvalue Problems},
  journal = {SIAM Review},
  volume = {15},
  number = {4},
  pages = {727--764},
  year = {1973},
  doi = {10.1137/1015095}
}

@article{xu2026transfer,
  author = {Xu, Tie and Wu, Shengdun and Luo, Junwen and Yang, Yao and Lin, Feng and Jiang, Junjie and Chen, Guozhang and Yang, Dongping},
  title = {Representation Transfer via Invariant Input-driven Continuous Attractors for Fast Domain Adaptation},
  journal = {Communications Biology},
  volume = {9},
  pages = {747},
  year = {2026},
  doi = {10.1038/s42003-026-09938-8}
}

@inproceedings{zuo2023temporal,
  author = {Zuo, Junfeng and Liu, Xiao and Wu, Ying Nian and Wu, Si and Zhang, Wenhao},
  title = {A Recurrent Neural Circuit Mechanism of Temporal-scaling Equivariant Representation},
  booktitle = {Advances in Neural Information Processing Systems},
  volume = {36},
  year = {2023},
  url = {https://proceedings.neurips.cc/paper_files/paper/2023/hash/c5e44243e16c9d61d3897ba1095f5f6c-Abstract-Conference.html}
}

@article{cotteret2025spatial,
  author = {Cotteret, Madison and Kymn, Christopher J. and Greatorex, Hugh and Ziegler, Martin and Chicca, Elisabetta and Sommer, Friedrich T.},
  title = {High-resolution Spatial Memory Requires Grid-cell-like Neural Codes},
  journal = {arXiv preprint arXiv:2507.00598},
  year = {2025},
  eprint = {2507.00598},
  archivePrefix = {arXiv},
  primaryClass = {cs.NE}
}

@inproceedings{deMouraUllrich2021,
  author = {de Moura, Leonardo and Ullrich, Sebastian},
  title = {The {Lean 4} Theorem Prover and Programming Language},
  booktitle = {Automated Deduction -- {CADE 28}},
  series = {Lecture Notes in Computer Science},
  volume = {12699},
  pages = {625--635},
  publisher = {Springer},
  year = {2021},
  doi = {10.1007/978-3-030-79876-5_37}
}

@inproceedings{mathlib2020,
  author = {{The mathlib Community}},
  title = {The {Lean} Mathematical Library},
  booktitle = {Proceedings of the 9th {ACM SIGPLAN} International Conference on Certified Programs and Proofs},
  pages = {367--381},
  publisher = {ACM},
  year = {2020},
  doi = {10.1145/3372885.3373824}
}

@misc{george2026torchlean,
  author = {George, Robert Joseph and Cruden, Jennifer and Adkisson, Will and Zhong, Xiangru and Zhang, Huan and Anandkumar, Anima},
  title = {{TorchLean}: Formalizing Neural Networks in {Lean}},
  year = {2026},
  eprint = {2602.22631},
  archivePrefix = {arXiv},
  primaryClass = {cs.MS},
  doi = {10.48550/arXiv.2602.22631}
}

@misc{cipollina2025hopfield,
  author = {Cipollina, Matteo and Karatarakis, Michail and Wiedijk, Freek},
  title = {Formalized Hopfield Networks and Boltzmann Machines},
  year = {2025},
  eprint = {2512.07766},
  archivePrefix = {arXiv},
  primaryClass = {cs.LG},
  doi = {10.48550/arXiv.2512.07766}
}

@article{wagner2026balancing,
  author = {Wagner, Margot and Chen, Yusi and Karuvally, Arjun and Cameron, Mia and Sejnowski, Terrence J.},
  title = {Balancing Stability and Flow in Hippocampal Networks via Inductive Bias and Learned Symmetry Breaking},
  journal = {Philosophical Transactions of the Royal Society B: Biological Sciences},
  volume = {381},
  number = {1954},
  pages = {20250260},
  year = {2026},
  doi = {10.1098/rstb.2025.0260}
}

\clearpage
\appendix
\section{Exact group-tangent transport}
\label{app:exact_proof}

\subsection{Proof of the controlled-map theorem}

\begin{proof}[Proof of \cref{thm:controlled_neutral}]
For any $\xi\in\g$, per-input equivariance gives
\[
F_u(\exp(s\xi)\cdot x)=\exp(s\xi)\cdot F_u(x).
\]
Differentiating at $s=0$ yields
\begin{equation}
 D F_u(x)A_x\xi=A_{F_u(x)}\xi.
 \label{eq:one_step_transport_app}
\end{equation}
Iterating \cref{eq:one_step_transport_app} along $x_n,\ldots,x_m$ proves
\cref{eq:exact_tangent_transport}. Equivariance implies stabilizer
monotonicity: if $g\cdot x_t=x_t$, then
$g\cdot x_{t+1}=F_{u_t}(g\cdot x_t)=F_{u_t}(x_t)=x_{t+1}$. Hence
$\ker A_{x_t}\subseteq\ker A_{x_{t+1}}$. Since every $A_x$ has rank $q$ on
$K$, the kernels have equal dimension and therefore coincide along the
trajectory.
The conclusion used below is equality of the isotropy Lie algebras
$\ker A_{x_t}$ along this trajectory. Under a proper action, full stabilizers
within one orbit-type stratum are conjugate and need not be literally equal as
embedded subgroups.

Write
\[
P_{m:n}:=D F_{u_{m-1}}(x_{m-1})\cdots D F_{u_n}(x_n).
\]
For $v\in E^G_{x_n}$, choose the unique
$\xi\in(\ker A_{x_n})^\perp$ such that $v=A_{x_n}\xi$. Kernel equality
makes $\xi\in(\ker A_{x_m})^\perp$, and exact transport gives
$P_{m:n}v=A_{x_m}\xi$. The nonzero-singular-value bounds therefore imply
\[
 \frac{c}{C}\norm v\le \norm{P_{m:n}v}\le\frac{C}{c}\norm v
\]
uniformly in $m-n$. Thus the operator norm and conorm of the restricted
cocycle are uniformly bounded, and every restricted singular-value growth rate
is squeezed between $(m-n)^{-1}\log(c/C)$ and
$(m-n)^{-1}\log(C/c)$ and tends to zero. For a prescribed nonstationary input
sequence, these uniform bounds still hold pathwise, but no stationary
Oseledets spectrum is asserted unless the required limits exist.
\end{proof}

\subsection{Stationary full-spectrum inclusion}

\begin{proof}[Proof of \cref{cor:stationary_cocycle}]
Write $Y=\Omega\times K$. If $\Theta$ is invertible modulo $\mu$, work on
$(Y,\mu,\Theta)$ itself. Otherwise let
$(\widehat Y,\widehat\mu,\widehat\Theta)$ be its natural extension, with
zeroth-coordinate projection $\pi_0$ satisfying
$\pi_0\widehat\Theta=\Theta\pi_0$ and $(\pi_0)_*\widehat\mu=\mu$. Define
$\widehat L(\widehat y)=L(\pi_0\widehat y)$ on the pulled-back tangent fibers and
$\widehat E^G_{\widehat y}=E^G_{\pi_0\widehat y}$. Then
\[
 \widehat L^{(n)}(\widehat y)=L^{(n)}(\pi_0\widehat y),
\]
so lifting changes neither forward products nor their Lyapunov exponents.
The chosen base transformation is now invertible; only the linear fiber maps
may remain noninvertible, as allowed by the semi-invertible theorem. All
statements below hold almost everywhere with respect to the relevant invariant
measure.

The rank-$q$ bundle $\widehat E^G$ is measurable by assumption. Exact
transport and constant rank give
$\widehat L(\widehat y)\widehat E_{\widehat y}^G
=\widehat E_{\widehat\Theta\widehat y}^G$, so the restriction is an invariant
linear cocycle. The pathwise bounds in \cref{thm:controlled_neutral} imply
that every nonzero restricted vector has exponent zero and that all $q$
restricted singular-value exponents are zero.

Choose a measurable complement of $\widehat E^G$. In the resulting adapted frames the
full cocycle has the block form
\[
 \begin{pmatrix}L_G&*\\0&L_Q\end{pmatrix},
\]
where $L_G$ is the restriction to $\widehat E^G$ and $L_Q$ is the induced quotient
cocycle. The integrability in \cref{eq:met_integrability} also controls these
finite-dimensional diagonal blocks. The invariant-subbundle property of the
semi-invertible multiplicative ergodic theorem, equivalently the spectrum
union property for this block-upper-triangular cocycle, places the $q$ zero
exponents of $L_G$ in the full spectrum with multiplicity
\citep{froyland2013semi}. Thus the full spectrum contains at least $q$ zeros;
the theorem does not exclude additional flow-neutral or transverse zeros.
Because the lifted and original forward products coincide under $\pi_0$, this
is also the full forward spectrum over $\mu$.
No statement about a stationary full spectrum is made for an arbitrary
prescribed control sequence.
\end{proof}

\begin{lemma}[Uniform nondegeneracy on a compact constant-rank stratum]
\label{lem:compact_non_degenerate}
Let $K$ be compact and $x\mapsto A_x$ continuous, where every
$A_x:\g\to T_xM$ has rank $q$. Then there exist $0<c\le C<\infty$ such that
all nonzero singular values of $A_x$ lie in $[c,C]$ for $x\in K$.
\end{lemma}

\begin{proof}
The largest singular value and the smallest nonzero singular value are
continuous on a constant-rank family in fixed local orthonormal frames. The
former attains a finite maximum on $K$, while rank $q$ makes the latter
positive pointwise and compactness makes its minimum positive. Equivalent
local trivializations give the same conclusion because transition maps are
uniformly bounded on compact overlaps.
\end{proof}

The explicit conditioning assumption in the main text is therefore redundant
for a smooth action on a compact constant-rank stratum, but it is useful
computationally: $C/c$ and the reducing-frame condition number quantify how a
small slope in adapted coordinates can appear in the ambient Euclidean metric.

\subsection{Autonomous-flow corollary and the ordinary flow-zero mode}

\begin{corollary}[Autonomous equivariant flow]
Let $\dot x=f(x)$ be a complete $C^1$ flow $\phi_t$ satisfying
$f(g\cdot x)=D(g)_x f(x)$. Under the compact constant-rank hypotheses of
\cref{thm:controlled_neutral},
\begin{equation}
 D\phi_t(x)A_x=A_{\phi_t(x)},
 \label{eq:flow_transport_app}
\end{equation}
and every nonzero group-tangent vector has zero direct logarithmic growth rate.
Applying the stationary hypotheses of \cref{cor:stationary_cocycle} to a
fixed positive-time map, the full Oseledets spectrum contains at least $q$
zero exponents, counted with multiplicity.
\end{corollary}

\begin{proof}
Equivariance of the vector field and uniqueness of solutions imply flow
equivariance. For every $\xi\in\g$,
\[
 \phi_t(\exp(s\xi)\cdot x)=\exp(s\xi)\cdot\phi_t(x).
\]
Differentiating this identity at $s=0$ gives
\[
 D\phi_t(x)\xi_M(x)=\xi_M(\phi_t(x)),
\]
which is \cref{eq:flow_transport_app}. Restrict $\xi$ to a fixed linear
complement of $\ker A_x$; its dimension is
$q=\dim\g-\dim\ker A_x=\dim(G/H)$, and the restricted infinitesimal action is
injective. The uniform nonzero-singular-value bounds then squeeze every
direct group-tangent growth rate to zero exactly as in
\cref{thm:controlled_neutral}. The full-spectrum conclusion is the separate
MET corollary, not a consequence of this differentiation alone. This is the
standard infinitesimal equivariance identity used in equivariant dynamics
\citep{golubitsky1988,krupa1990}.
\end{proof}

For an invertible flow, the full stabilizer is preserved exactly along an
orbit: $G_x\subseteq G_{\phi_t(x)}$ follows from equivariance, and applying the
same argument to $\phi_{-t}$ gives equality. The ordinary autonomous-flow zero
mode is generated by $f(x)$. It may coincide with one group tangent at a
relative equilibrium, but it vanishes on a fixed-point continuous attractor.
Therefore, a single near-zero exponent is not sufficient evidence for a
protected memory mode. Product groups with $q>1$, direct analytical tangents,
and principal-angle alignment resolve this ambiguity.

\subsection{Why a finite symmetry group is different}
A finite group can enforce parameter sharing, block structure, and spectral
degeneracies, but its Lie algebra is zero. It therefore provides no nonzero
fundamental vector fields $\xi_M$ and does not imply a continuously
parameterized neutral tangent through \cref{eq:one_step_transport_app}. A
finite circulant grid has exact integer-roll $C_N$ equivariance but generally
fails equivariance under a generic continuous phase shift. Such grids remain
useful approximations and null controls, but they are not theorem models for
continuous memory.

\section{Additional exact-geometry diagnostics}
\label{app:exact_diagnostics}

The overview in \cref{tab:overview} summarizes the role of these additional
diagnostics. The numerical thresholded count is always
secondary to direct tangent transport and analytical subspace alignment.

\begin{table}[h]
\centering
\caption{Representative exact-symmetry diagnostics. Values are maxima over
the corresponding evaluations fixed before analysis unless otherwise stated. The learned
$S^1$ step residual is inherited from the original float32 task pipeline; the
native-coordinate perturbation analysis reloads the same checkpoints in float64
and therefore reports
roundoff-scale exact-symmetry residuals. The values reflect different
floating-point evaluation regimes.}
\label{tab:exact_diagnostics}
\begin{tabular}{@{}p{0.66\linewidth}r@{}}
\toprule
Diagnostic & Value\\
\midrule
Coupled-irrep vector-field equivariance error & $6.97\times10^{-16}$\\
Coupled-irrep direct group-tangent exponent magnitude & $5.30\times10^{-13}$\\
Coupled-irrep tangent-covariance angle & $1.48\times10^{-6}$ degrees\\
Broken coupled-irrep equivariance error & $1.29\times10^{-2}$\\
Learned $S^1$ maximum exact step-equivariance error & $3.19\times10^{-8}$\\
Learned $S^1$ mean zero-input direct tangent exponent & $-2.28\times10^{-6}$\\
Learned $S^1$ mean zero-input principal angle & $2.97\times10^{-2}$ degrees\\
\bottomrule
\end{tabular}
\end{table}

\section{Blockwise invariant-graph criterion and rate bounds}
\label{app:block_proof}

\subsection{Reducing an upper-triangular exact map}
\label{app:shear}
Let $X\oplus N$ be a finite-dimensional tangent/normal splitting and suppose
an exact local time map has derivative
\begin{equation}
 T_0=\begin{pmatrix}I&H\\0&B\end{pmatrix}.
 \label{eq:upper_triangular_app}
\end{equation}
The lower-left zero follows from invariance of the group tangent; the
upper-right block depends on the chosen complement.

\begin{lemma}[Complement-shear reduction]
\label{lem:shear}
If $I-B$ is invertible, set
\begin{equation}
 C=-H(I-B)^{-1},\qquad
 S=\begin{pmatrix}I&C\\0&I\end{pmatrix}.
 \label{eq:shear_def}
\end{equation}
Then
$S^{-1}T_0S=\diag(I,B)$. If $\norm B<1$, $I-B$ is invertible by the Neumann
series.
\end{lemma}

\begin{proof}
Direct multiplication gives the transformed upper-right block
$H+C(I-B)=0$. The diagonal blocks and lower-left zero are unchanged.
\end{proof}

The experiment records the map duration, the normal gap, the shear reduction
residual, and $\kappa_S=\norm S\norm{S^{-1}}$. This prevents an apparently
small adapted-coordinate angle from being presented without its ambient
conditioning cost.

\begin{lemma}[Directed gap equals the sine of the largest principal angle]
\label{lem:directed_gap_principal_angle}
Let $U,V$ be equal $q$-dimensional subspaces of a finite-dimensional Euclidean
space, with $q\ge1$. If $P_V$ is the orthogonal projector onto $V$ and
$0\le\theta_1\le\cdots\le\theta_q\le\pi/2$ are their principal angles, then
\begin{equation}
 \delta(U,V):=\sup_{\substack{u\in U\\\norm u=1}}
 \operatorname{dist}(u,V)
 =\norm{P_{V^\perp}|_U}
 =\sin\theta_q.
 \label{eq:directed_gap_principal_angle}
\end{equation}
Thus $\arcsin\delta(U,V)=\theta_q$, the conventional largest principal angle.
\end{lemma}

\begin{proof}
The nearest point in $V$ to $u$ is $P_Vu$, so
$\operatorname{dist}(u,V)=\norm{(I-P_V)u}$, proving the first equality. Let
$Q_U,Q_V$ have orthonormal columns spanning $U,V$. By the singular-value
definition of principal angles, the singular values of $Q_V^\top Q_U$ are
$\cos\theta_i$ \citep{knyazevArgentati2002}. Moreover,
\[
 Q_U^\top(I-Q_VQ_V^\top)Q_U
 =I-(Q_V^\top Q_U)^\top(Q_V^\top Q_U).
\]
Its eigenvalues are therefore $\sin^2\theta_i$. Hence the operator norm of
$(I-P_V)Q_U$, equivalently of $P_{V^\perp}|_U$, is
$\max_i\sin\theta_i=\sin\theta_q$.
\end{proof}

\subsection{Proof of Proposition~\ref{prop:block_certificate}}

For a graph $\{(x,Lx):x\in X\}$, the image under \cref{eq:defect_blocks} is a
graph whenever the center factor is invertible, with graph transform
\begin{equation}
 \Phi(L)=\bigl(E_{NG}+(B+E_{NN})L\bigr)
 \bigl(I+E_{GG}+E_{GN}L\bigr)^{-1}.
 \label{eq:graph_transform_app}
\end{equation}
If $\norm L\le r$ and $a+br<1$, the inverse exists and
\begin{equation}
 \norm{\Phi(L)}
 \le \frac{c+(\rho+e)r}{1-a-br}.
 \label{eq:graph_bound_app}
\end{equation}
The right-hand side is at most $r$ exactly when
\begin{equation}
 br^2-(d-a-e)r+c\le0.
 \label{eq:quadratic_graph}
\end{equation}
Under $D=d-a-e>0$ and $D^2\ge4bc$, the smaller nonnegative root is
\cref{eq:rstar}; the displayed form is numerically stable when $bc$ is small.
For $b=0$, \cref{eq:quadratic_graph} reduces to $c-Dr\le0$ and $r_*=c/D$.
The graph ball is compact and convex in the finite-dimensional operator space,
and $\Phi$ is continuous on it, so Brouwer's theorem supplies at least one
fixed graph. This existence argument does not claim uniqueness.

On a fixed graph, write $J_Lx=(x,Lx)$ and
$M_L=I+E_{GG}+E_{GN}L$. The restricted map is conjugate to its
center-coordinate representation,
$T|_{\Gamma_L}=J_LM_LJ_L^{-1}$. The distance from $M_L$ to identity is at
most $s=a+br_*$, so its norm is at most $1+s$ and its conorm at least $1-s$.
In center coordinates, $k$ iterates of the same local map obey
\[
 (1-s)^k\norm x\le\norm{M_L^kx}\le(1+s)^k\norm x.
\]
Passing to the ambient graph norm introduces only the fixed norm-equivalence
factor
\[
 \kappa(J_L)^{-1}(1-s)^k\norm v
 \le \norm{T^kv}
 \le \kappa(J_L)(1+s)^k\norm v,
 \qquad v\in\Gamma_L.
\]
Because $\kappa(J_L)$ is independent of $k$, its logarithm vanishes after
division by physical time $k\tau$ and passage to the asymptotic rate. This
proves \cref{eq:exp_bound_main}. The argument concerns repeated iterations of
the same selected local linear map; it does not establish a time-varying
graph-sequence theorem.

To obtain the ambient-angle bound, a reducing-coordinate graph vector
$(x,Lx)$ lies at distance at most $\norm S\,r_*\norm x$ from $S(X,0)$, while
\[
\norm{S(x,Lx)}\ge\norm{S^{-1}}^{-1}\norm{(x,Lx)}
\ge\norm{S^{-1}}^{-1}\norm x.
\]
Taking the supremum over unit graph vectors bounds the directed gap from the
ambient graph to the exact center plane by $\min\{1,\kappa_Sr_*\}$.
The two spaces have the same dimension, so
\cref{lem:directed_gap_principal_angle} gives
$\sin\alpha\le\min\{1,\kappa_Sr_*\}$ for their largest principal angle.

\subsection{Perturbative orders and one-way leakage}
With $a=e=0$, write $D=d$ and expand the smaller root:
\begin{align}
 r_*&=\frac{2c}{d+d\sqrt{1-4bc/d^2}}
     =\frac{c}{d}+\frac{bc^2}{d^3}+O((bc)^2c/d^5),\\
 s&=br_*=\frac{bc}{d}+O(b^2c^2/d^3).
\end{align}
Thus, with $d>0$ fixed as $\epsilon\to0$, $b,c=O(\epsilon)$ imply
$s=O(\epsilon^2/d)$. If $c=0$, then $r_*=0$
and upper leakage $b$ alone leaves the exact center graph invariant. If
$b=0$, the graph tilts by at most $c/D$, but the center map remains
$I+E_{GG}$; lower leakage alone does not change the center exponent in the
constant local block model. Normal-only perturbations similarly leave the
center exponent at numerical zero, which is why bound/measured ratios are not
reported for that family.

\subsection{Local affine residual realization}
\label{app:local_realization}

\begin{lemma}[Realizing a prescribed local Jacobian defect]
\label{lem:residual_realization}
Let $F_0:\R^n\to\R^n$ be $C^1$, let $x_*$ be a selected orbit point, and let
$\Delta\in\R^{n\times n}$. There exists a compactly supported $C^1$ residual
branch $R$ such that
\[
 R(x_*)=0,\qquad DR(x_*)=\Delta.
\]
Hence $F_\epsilon=F_0+\epsilon R$ has the same selected state value and local
Jacobian $DF_0(x_*)+\epsilon\Delta$.
\end{lemma}

\begin{proof}
Choose a $C^1$ bump function $\chi$ equal to one in a neighborhood of $x_*$
and define
$R(x)=\chi(x)\Delta(x-x_*)$. Then $R(x_*)=0$ and $DR(x_*)=\Delta$.
\end{proof}

This lemma justifies calling the scalar selected-map study a causal local
recurrent residual
intervention. It does not imply that an arbitrary weight perturbation has the
same block routing, nor that one local branch preserves any global symmetry or
orbit away from the selected point.

\section{Finite-horizon bounds along controlled recurrent trajectories}
\label{app:finite_horizon_candidate}

\paragraph{Scope and coherent coordinates.}
Fix $n<m$. Let $X\cong\mathbb{R}^q$ and $N\cong\mathbb{R}^p$ be fixed
Euclidean model spaces over the prescribed finite horizon, and let
$S_t:X\oplus N\rightarrow T_{x_t}M$ be invertible for every endpoint. The
same endpoint frame is used as the target frame at step $t-1$ and the source
frame at step $t$. The coordinate cocycle is
\begin{equation}
 T_t=S_{t+1}^{-1}D F_{u_t}(x_t)S_t.
 \label{eq:finite_frame_cocycle_candidate}
\end{equation}
Let $A_t:\mathfrak g\to T_{x_t}M$ be the infinitesimal-action map. Exact
per-input equivariance gives
\begin{equation}
 D F_{u_t}(x_t)A_t=A_{t+1}.
 \label{eq:finite_horizon_action_transport_candidate}
\end{equation}
Assume $\operatorname{rank}A_t=q$ along the trajectory. Then
$\ker A_t\subseteq\ker A_{t+1}$, and equality of dimensions makes these
kernels one common subspace $\mathfrak h_*\subset\mathfrak g$. Fix once and for all a
complement $V$ with $\mathfrak g=\mathfrak h_*\oplus V$ and a basis of $V$ (equivalently,
a fixed basis of the vector-space quotient $\mathfrak g/\mathfrak h_*$; no quotient Lie
algebra is asserted). Use the $q$ full-rank columns of
$A_t|_V$ at every endpoint. Passing all raw Lie-algebra columns is valid only
for a locally free action; with a nontrivial stabilizer those columns are rank
deficient. Endpoint-wise QR normalization generally inserts nonidentity
transition matrices, which cannot be silently replaced by an identity center
block. Appending orthogonal normal complements first gives
\begin{equation}
 \begin{pmatrix}I&C_t\\0&B_t\end{pmatrix}.
\end{equation}
With $R_t=\left(\begin{smallmatrix}I&K_t\\0&I\end{smallmatrix}\right)$,
direct multiplication gives
\begin{equation}
 R_{t+1}^{-1}
 \begin{pmatrix}I&C_t\\0&B_t\end{pmatrix}R_t
 =\begin{pmatrix}I&K_t+C_t-K_{t+1}B_t\\0&B_t\end{pmatrix}.
\end{equation}
Thus the terminal anchor $K_m=0$ and backward recurrence
$K_t=K_{t+1}B_t-C_t$ remove $C_t$ on the declared horizon. Explicitly,
\begin{equation}
 K_t=-\sum_{j=t}^{m-1}C_j B_{j-1}\cdots B_t,\qquad
 \norm{K_t}\le\sum_{j=t}^{m-1}\norm{C_j}
 \prod_{\ell=t}^{j-1}\rho_\ell,
 \label{eq:finite_horizon_shear_bound_candidate}
\end{equation}
where the product is the identity when $j=t$. Hence, if
$\rho_t\le\bar\rho<1$ and $\norm{C_t}\le\bar C$, then
$\norm{K_t}\le\bar C(1-\bar\rho^{m-t})/(1-\bar\rho)$; without a
uniform gap the shear may grow with the horizon. The shear leaves the action
columns unchanged but depends on the future reference sequence; the resulting
reduction is therefore offline/trajectory-conditioned unless the future
controls and reference maps are known. Its horizon and gap dependence is
retained in $S_t$ and $\kappa(S_t)$ rather than discarded.

In these reducing frames, write explicitly
\begin{equation}
 T_t=\begin{pmatrix}
 I+E_{GG,t}&E_{GN,t}\\
 E_{NG,t}&B_t+E_{NN,t}
 \end{pmatrix},
\quad
\begin{matrix}
 \norm{E_{GG,t}}\le a_t,&\norm{E_{GN,t}}\le b_t,\\
 \norm{E_{NG,t}}\le c_t,&\norm{E_{NN,t}}\le e_t,
\end{matrix}
\label{eq:finite_horizon_blocks_candidate}
\end{equation}
with $\norm{B_t}\le\rho_t<1$.

\paragraph{Proof of \cref{prop:finite_horizon_candidate}.}
For $J_Lv=(v,Lv)$, block multiplication gives
\begin{align}
 M_tv&=(I+E_{GG,t}+E_{GN,t}L_t)v,\\
 L_{t+1}&=\{E_{NG,t}+(B_t+E_{NN,t})L_t\}M_t^{-1}.
 \label{eq:finite_graph_identity_candidate}
\end{align}
If $\norm{L_t}\le r_t$, then
$\norm{M_t-I}\le a_t+b_tr_t=s_t<1$, so $M_t$ is invertible and
$\norm{M_t^{-1}}\le(1-s_t)^{-1}$. Taking norms in
\cref{eq:finite_graph_identity_candidate} gives exactly
\cref{eq:finite_horizon_recursion_candidate}; induction proves graph transport
and $\norm{L_t}\le r_t$. The same multiplication gives
$T_tJ_{L_t}=J_{L_{t+1}}M_t$.

Each $M_t:X\to X$ has conorm at least $1-s_t$ and norm at most
$1+s_t$. For the chronological center product
$P_{m:n}=M_{m-1}\cdots M_n$, supermultiplicativity of conorm and
submultiplicativity of norm give
\begin{equation}
 \prod_{t=n}^{m-1}(1-s_t)\le\sigma_{\min}(P_{m:n})
 \le\norm{P_{m:n}}\le\prod_{t=n}^{m-1}(1+s_t).
 \label{eq:finite_horizon_product_candidate}
\end{equation}
Multiplying the graph identities telescopes to
\begin{equation}
 D F_{m:n}G_n=G_mP_{m:n},\qquad G_t=S_tJ_{L_t},
\end{equation}
where $D F_{m:n}=D F_{u_{m-1}}(x_{m-1})\cdots D F_{u_n}(x_n)$.
Thus the lower ambient multiplier is bounded by
$[\sigma_{\min}(G_m)/\norm{G_n}]\prod_t(1-s_t)$ and the upper multiplier by
$[\norm{G_m}/\sigma_{\min}(G_n)]\prod_t(1+s_t)$. Directly computable endpoint
bounds are
\begin{equation}
 \sigma_{\min}(G_t)\ge\norm{S_t^{-1}}^{-1},\qquad
 \norm{G_t}\le\norm{S_t}\sqrt{1+r_t^2}.
\end{equation}
Consequently, for every nonzero $v\in X$, positive step durations $\tau_t$,
and $\Delta\tau=\sum_{t=n}^{m-1}\tau_t$,
\begin{align}
 \frac{1}{\Delta\tau}\log\!\left[
 \frac{\sigma_{\min}(G_m)}{\norm{G_n}}
 \prod_{t=n}^{m-1}(1-s_t)\right]
 &\le \frac{1}{\Delta\tau}
 \log\frac{\norm{D F_{m:n}G_nv}}{\norm{G_nv}},\\
 \frac{1}{\Delta\tau}
 \log\frac{\norm{D F_{m:n}G_nv}}{\norm{G_nv}}
 &\le \frac{1}{\Delta\tau}\log\!\left[
 \frac{\norm{G_m}}{\sigma_{\min}(G_n)}
 \prod_{t=n}^{m-1}(1+s_t)\right].
 \label{eq:finite_horizon_ambient_rates_candidate}
\end{align}
Finally, the distance from $S_t(v,L_tv)$ to the reference center plane is at
most $\norm{S_t}r_t\norm v$, whereas its norm is at least
$\sigma_{\min}(S_t)\norm v$. Taking the worst normalized distance bounds the
directed gap by $\min\{1,\kappa(S_t)r_t\}$. Since the transported graph and
reference center plane both have dimension $q$,
\cref{lem:directed_gap_principal_angle} gives
$\sin\alpha_t\le\min\{1,\kappa(S_t)r_t\}$ for the conventional largest
principal angle. Only endpoint frame factors remain because the same coherent
sequence appears on both sides of every adjacent map. This proves the
finite-horizon statement for the one specified initial graph. It does not
select that graph, prove existence or uniqueness of a bi-infinite invariant
bundle, or identify an asymptotic Lyapunov spectrum. Such a continuation would
require stronger exponential-dichotomy or normal-hyperbolicity hypotheses and
belongs to the classical persistence framework
\citep{fenichel1971,hirschPughShub1977}.

\paragraph{Formalization.}
The finite-sequence graph transport, chronological product bounds, endpoint
factors, terminal shear, and leakage estimate are machine-checked in Lean 4.
Lean proves directed-gap bounds; \cref{lem:directed_gap_principal_angle}
identifies the conventional largest principal angle in the manuscript. The
endpoint-indexed construction does not require smooth dependence on the base
point. \Cref{app:formal_correspondence} gives the complete scope and separates
these results from the general Lie-orbit and multiplicative-ergodic interfaces.

\paragraph{Constant-map and perturbative limits.}
For a repeated map, initialize the scalar recurrence at $r_*$ from
\cref{eq:rstar}. Proposition~\ref{prop:block_certificate} supplies an invariant
graph in the corresponding closed ball. Because
\[
c+(\rho+e)r_*=r_*(1-a-br_*),
\]
the scalar radius remains $r_*$, $s_t=a+br_*$, and the products reduce to
$(1-s)^k$ and $(1+s)^k$.
Starting from zero instead describes a transient forward graph.

The leakage order can be stated quantitatively. Suppose $r_n=0$,
$a_t=e_t=0$, $\rho_t\le1-d$ with $0<d\le1$, and $\epsilon\ge0$. Uniformly
in $t$, let $b_t\le B\epsilon$ and $c_t\le C\epsilon$ for constants
$B,C\ge0$. If
$4BC\epsilon^2\le d^2$, then induction with $R=2C\epsilon/d$ gives
\begin{equation}
 r_t\le R,\qquad
 1-b_tr_t\ge1-d/2>0,\qquad
 s_t=b_tr_t\le\frac{2BC\epsilon^2}{d}
 \label{eq:finite_horizon_leakage_corollary}
\end{equation}
for every step on the declared sequence. Indeed,
$C\epsilon+(1-d)R\le R(1-B\epsilon R)$ is equivalent to
$C\epsilon\le dR-B\epsilon R^2$, which follows from the stated smallness
condition. Thus bidirectional leakage is second order in the center-coordinate
rate bound under these uniform hypotheses, while direct center damage
$a_t=O(\epsilon)$ enters $s_t$ at first order. Ambient rates also retain the
endpoint chart terms in \cref{eq:finite_horizon_ambient_rates_candidate}; no
asymptotic conclusion follows without controlling those terms and the frame
sequence as the horizon grows.

\paragraph{Controlled-trajectory evaluation.}
The 32-step macro-map protocol was fixed before held-out evaluation. Eight
fresh base-model seeds, each with three nested input
trajectories, yielded 216 valid structured-residual cases with coordinate,
ambient, and angle coverage all equal to one. All 24 deliberately failed
hypothesis cases were retained and detected. Slopes were first averaged over
nested trajectories and then summarized across the eight base-model seeds:
$0.916\pm0.095$ for direct damage and $2.040\pm0.092$ for leakage (mean
$\pm$ SEM). They are consistent with the predicted first- and second-order
routes over this tested finite-horizon regime; the direct estimate is visibly
noisier. No nonlinear endpoint metric entered pilot selection. These
structured time-varying residuals test pre-rollout predictions, not a universal
optimizer, quantizer, or hardware-error distribution.

\begin{figure}[t]
\centering
\includegraphics[width=\linewidth]{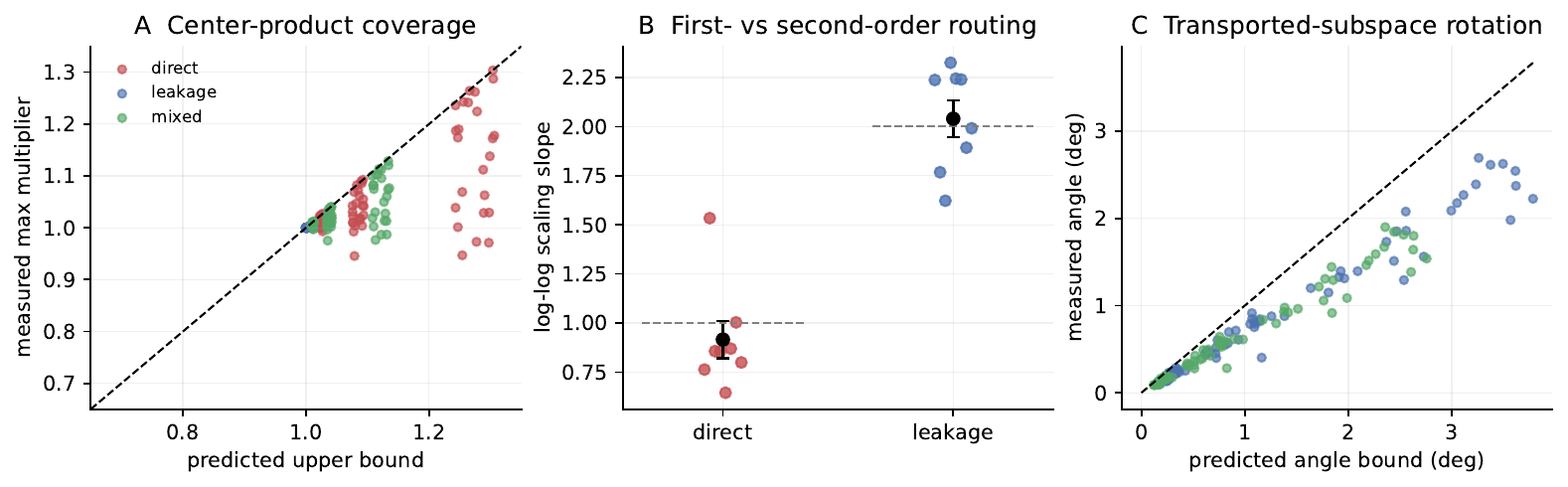}
\caption{\textbf{Finite-horizon results along controlled trajectories.}
Held-out diagnostics report predicted coordinate, ambient, and angle bounds,
retained assumption-violation controls, and route-dependent scaling.
Trajectories and timesteps are nested within eight base-model seeds. This is a
finite-horizon, offline/trajectory-conditioned test of structured residuals;
it is not an asymptotic Lyapunov-spectrum result, a bi-infinite-bundle theorem,
an online guarantee, or a universal perturbation model.}
\end{figure}

\section{Selected-map perturbation study}
\label{app:exp37}

\subsection{Models, maps, and interventions}
The experiment fixed six coupled-irrep recurrent instances and all six
learned exact $S^1$ checkpoints. For each model, an exact local time map at a
fixed orbit point is reduced using Lemma~\ref{lem:shear}. Perturbations are inserted
in the reducing blocks and transformed back to ambient coordinates. The
protected bundle is scalar ($q=1$) in every instance of this study; the
proposition itself permits general finite $q$. The intervention families are:
\begin{itemize}[leftmargin=1.5em,itemsep=0.1em]
\item direct tangent: $E_{GG}\ne0$ only;
\item bidirectional leakage: $E_{GN},E_{NG}\ne0$ with $E_{GG}=E_{NN}=0$;
\item normal only: $E_{NN}\ne0$ only;
\item mixed: all blocks active under the routing rule fixed before evaluation.
\end{itemize}
Seven scales
$\epsilon\in\{10^{-5},3\!\times\!10^{-5},10^{-4},3\!\times\!10^{-4},
10^{-3},3\!\times\!10^{-3},10^{-2}\}$ are evaluated. Directional replication
creates nested rows within a model and scale; statistical summaries therefore
first aggregate to the model level.

\begin{table}[h]
\centering
\caption{Local map and frame ranges for the scalar perturbation study.}
\label{tab:exp37_map}
\begin{tabular}{@{}lcccc@{}}
\toprule
Family & Models & Map duration $\tau$ & Normal gap $d$ & $\kappa_S$\\
\midrule
Coupled-irrep RNN & 6 & $0.16$ (five), $0.32$ (one) & $0.1206$--$0.2086$ & $4.7435$\\
Learned equivariant cell & 6 & $0.8$ & $0.1682$--$0.1741$ & $1.0000$\\
\bottomrule
\end{tabular}
\end{table}

The near-identical direct-family slopes and their tiny SEM are consequences of
the shared analytic intervention, which imposes the leading linear term. They
validate implementation of that controlled route but should not be read as
evidence that unrelated implementation defects have a universal first-order
coefficient.

\subsection{Conditions and bound coverage}
Of 1,608 total rows, 12 are zero-defect references and 1,596 are nonzero. The
sufficient conditions hold in 1,572 nonzero rows. The only failing nonzero
rows are the 24 learned mixed-direction cases at $\epsilon=10^{-2}$. A total
of 1,476 rows have informative exponent and ambient-angle bounds. The additional
uninformative rows are 48 coupled leakage and 48 coupled mixed cases at
$\epsilon=10^{-2}$. Among all rows satisfying the conditions, the observed counts of exponent,
invariant-graph-radius, and angle violations are each zero.

These are coverage counts over a structured test set, not independent
Bernoulli trials. Scaling uncertainties use the 12 model instances rather
than treating their perturbation rows as independent replicates.

\begin{table}[h]
\centering
\caption{Perturbative slopes, fit separately within each model and summarized
across six models per family.}
\label{tab:exp37_slopes}
\begin{tabular}{@{}llccc@{}}
\toprule
Model family & Defect family & Mean slope & SEM & Mean $R^2$\\
\midrule
Coupled-irrep & direct tangent & $1.000575$ & $2.5\times10^{-12}$ & $0.9999998$\\
Learned cell & direct tangent & $1.000575$ & $5.4\times10^{-13}$ & $0.9999998$\\
Coupled-irrep & bidirectional leakage & $2.003847$ & $0.000131$ & $0.9999951$\\
Learned cell & bidirectional leakage & $2.023140$ & $0.000729$ & $0.9998185$\\
\bottomrule
\end{tabular}
\end{table}

\begin{table}[h]
\centering
\caption{Sharpness of nontrivial physical-time exponent bounds. Ratios are
computed only when the measured exponent is above the numerical floor.}
\label{tab:exp37_sharpness}
\begin{tabular}{@{}llccc@{}}
\toprule
Model family & Defect family & Median bound/measured & 10th pct. & 90th pct.\\
\midrule
Coupled-irrep & direct tangent & $1.000$ & $1.000$ & $1.000$\\
Coupled-irrep & leakage & $1.468$ & $1.329$ & $1.698$\\
Coupled-irrep & mixed & $1.005$ & $1.000$ & $1.230$\\
Learned cell & direct tangent & $1.000$ & $1.000$ & $1.000$\\
Learned cell & leakage & $1.036$ & $1.022$ & $1.116$\\
Learned cell & mixed & $1.000$ & $1.000$ & $1.013$\\
\bottomrule
\end{tabular}
\end{table}

\section{Native-parameter and ambient-residual perturbations}
\label{app:exp39}

The scalar selected-map study specifies defects in reducing coordinates
to isolate causal routing. The native-coordinate perturbation analysis asks how
perturbations written in the recurrent model's own coordinates populate those
blocks. It uses all six learned exact $S^1$
checkpoints and the same zero-input orbit point and selected time maps.

\paragraph{Exact-parameterization control.}
Every trainable tensor is perturbed by independent Gaussian noise at five
relative scales from $10^{-4}$ to $10^{-2}$, with three draws per scale and
checkpoint. The per-tensor noise standard deviation is the requested scale
times the larger of that tensor's RMS value and $10^{-3}$. Because the
architecture still constructs coefficients only from invariants, these 90
weight perturbations remain exactly equivariant. The maximum sampled
step-equivariance error is $1.121\times10^{-16}$, the maximum analytical
tangent-transport residual is $1.241\times10^{-16}$, and the largest measured
protected-exponent magnitude is $8.33\times10^{-16}$. Thus ordinary parameter
error inside a hard-equivariant parameterization is not itself a symmetry
defect, although it can change transverse dynamics.

\paragraph{Ambient residual parameters.}
For each checkpoint we draw 20 unit-operator-norm matrices from each of three
families: dense Gaussian, 10\%-dense sparse Gaussian, and rank-one outer
products. At the selected state $x_*$, the local residual branch is
$R_P(x)=\chi(x)P(x-x_*)$ as in Lemma~\ref{lem:residual_realization}. Seven
scales give 2,520 nonzero cases. The residual is specified in ambient state
coordinates; only afterward is its induced selected-time-map defect projected
into $(E_{GG},E_{GN},E_{NG},E_{NN})$. The bound is therefore evaluated without
using the measured perturbed invariant mode.

Of 2,520 cases, 2,515 satisfy the conditions and 2,514 have informative
exponent and angle bounds. No case satisfying the conditions violates the exponent, graph, or
angle bound. The five failing and one additional angle-uninformative cases all
occur at $\epsilon=10^{-2}$ and are retained. The block-resolved bound is more
descriptively associated with the measured displacement than the scalar local
generator-commutator norm (log--log Pearson $r=0.834$ versus $0.533$), but
these correlations are not inferential because scales and directions are
nested within checkpoint.

\begin{table}[h]
\centering
\caption{Native-coordinate residual routing. Direct-route fractions compare
$a$ with the pre-rollout leakage proxy $bc/d$. Slopes are first summarized
over directions within checkpoint and then reported as mean $\pm$ SEM over
six checkpoints; sparse structural zeros produce a mixture of perturbative
orders rather than one universal slope.}
\label{tab:exp39}
\begin{tabular}{@{}lcccc@{}}
\toprule
Residual family & Cases & Valid & Direct-route fraction & Checkpoint slope, mean $\pm$ SEM\\
\midrule
Dense & 840 & 99.64\% & 96.90\% & $1.0028\pm0.0031$\\
Sparse & 840 & 99.76\% & 40.83\% & $1.9117\pm0.2715$\\
Rank one & 840 & 100.00\% & 100.00\% & $1.0001\pm0.0016$\\
\bottomrule
\end{tabular}
\end{table}

\begin{figure}[h]
\centering
\includegraphics[width=0.96\linewidth]{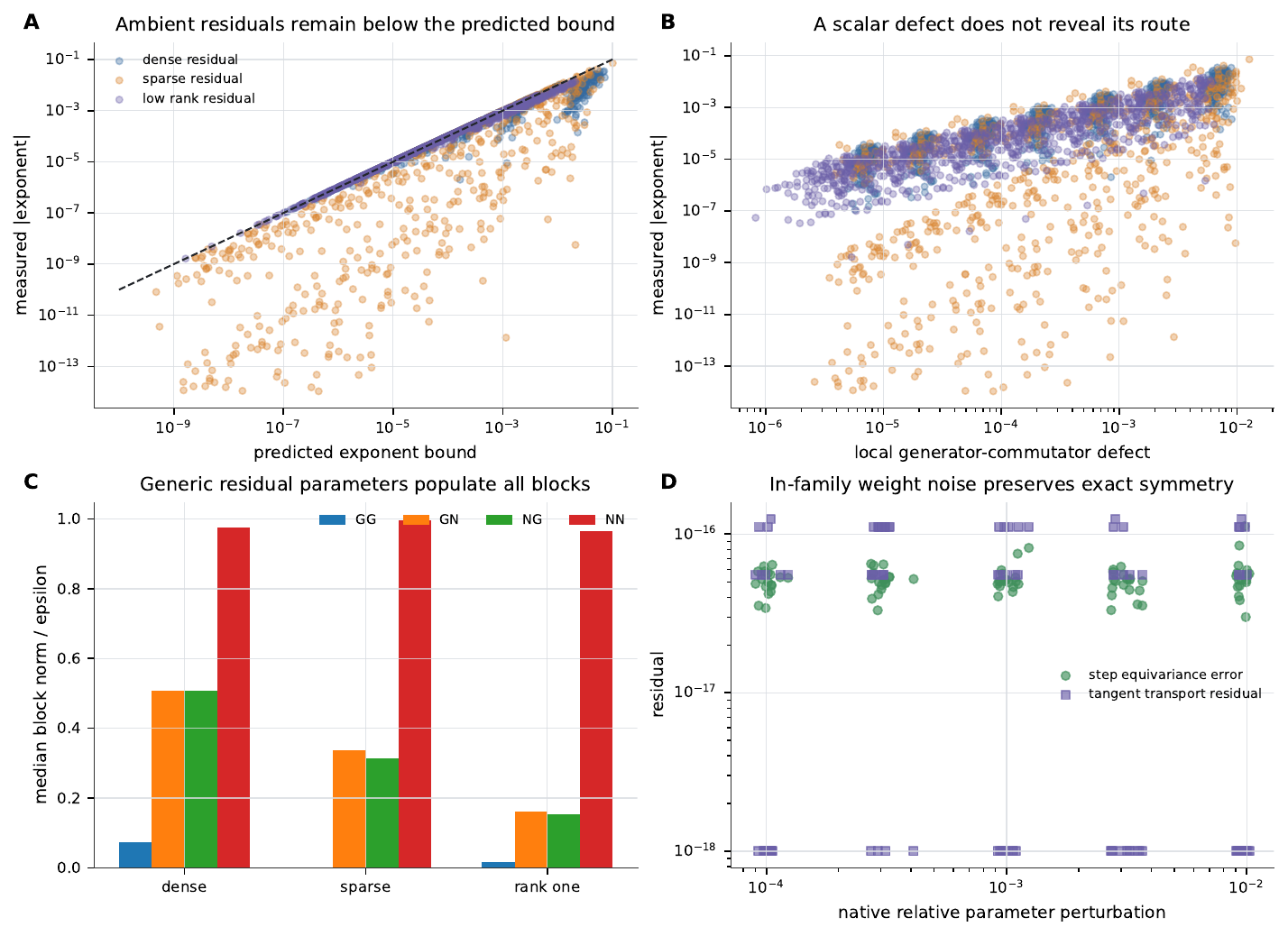}
\caption{\textbf{Parameter and implementation-defect analysis.} Ambient dense,
sparse, and low-rank residual parameters are evaluated only after their
induced block norms have been measured. Cases satisfying the conditions remain below the
predicted exponent bound; generic residuals populate multiple blocks, so a
scalar local commutator does not identify the degradation route. In contrast,
native weight noise inside the exact parameterization preserves step
equivariance and analytical tangent transport at numerical precision. This is
a local residual-family analysis, not a learned or hardware-specific defect
distribution.}
\label{fig:exp39}
\end{figure}

\section{Two-coordinate architecture-breaking rollout bridge}
\label{app:exp40}

The two-coordinate rollout bridge tests whether the general-$q$ implementation of
\cref{prop:block_certificate} remains informative for a learned $q=2$ memory
and for a defect applied to the recurrent update rather than directly to its
reducing blocks. We hold fixed the six condition-one $T^2$ checkpoints from the
prescribed-geometry study and select the memory state with phases
$(0.37,-0.41)$. For each checkpoint the exact one-step Jacobian has two
identity group tangents. Eight recurrent steps give the selected map; normal
gaps range from $0.1646$ to $0.1919$, the reducing-frame condition number is
one, and fixed-state, identity, invariance, and reduction residuals are at
numerical precision.

The broken update is
\begin{equation}
 F_{\epsilon,P}(x,u)=F_0(x,u)+\epsilon P(x-x_*),
 \label{eq:exp40_adapter}
\end{equation}
where $\norm P_2=1$ and $P$ is a channelwise gain-mismatch matrix, a rank-two
negative-semidefinite adapter, or a dense negative-semidefinite residual
adapter. The anchor keeps $x_*$ fixed, so its selected-map Jacobian is computed
exactly from $(D F_0(x_*)+\epsilon P)^8$. These ambient-coordinate adapters
break $T^2$ equivariance and generically populate several reducing blocks.
They are representative tying/residual errors, not a measured distribution
from a specific optimizer, quantizer, or device.

For a locally pinning continued tangent plane, let $\zeta_1,\zeta_2$ be the
selected-map eigenvalues for map duration $\tau$, and define the two positive
decay rates by $\gamma_j=-\tau^{-1}\log|\zeta_j|$. Order them as
$\gamma_{\rm fast}\geq\gamma_{\rm slow}>0$. The rollout bridge reports the
scalar spectral half-life
\begin{equation}
 \tau_{1/2}^{\rm geom}
 =\sqrt{\frac{\log 2}{\gamma_{\rm fast}}
              \frac{\log 2}{\gamma_{\rm slow}}},
 \label{eq:exp40_geometric_half_life}
\end{equation}
the geometric mean of the fast- and slow-mode half-lives. This scalar is a
declared summary for comparison with the median crossing time; the two modal
half-lives remain in the released table.

For that comparison, 16 probes begin at Euclidean radius $0.08$ in wrapped
two-phase coordinates around the fixed memory state. A probe's measured
half-life is the first time its radius is at most $0.04$, and the reported case
value is the median crossing time across probes; censoring is reported
separately.

\begin{table}[h]
\centering
\caption{Summary of the two-coordinate rollout study.
Coverage is evaluated only when the sufficient conditions hold.
Large-scale failures are retained.}
\label{tab:exp40}
\begin{tabular}{@{}lr@{}}
\toprule
Diagnostic & Result\\
\midrule
Fixed checkpoints / protected dimension & $6$ / $q=2$\\
Cases through $\epsilon=10^{-2}$ & $90/90$ satisfy conditions\\
Boundary cases at $\epsilon=3\times10^{-2}$ & $0/18$ satisfy conditions\\
Exponent / graph-radius / angle coverage & $100\%/100\%/100\%$\\
Bound--measured local displacement correlation & $0.9911$\\
Monotone global endpoint error units & $18/18$\\
Fully uncensored seed--adapter--scale cases & $41$ (16 probes each)\\
Primary log correlation; median ratio & $0.9996$; $1.0119$\\
Finite-median sensitivity cases & $56$ (15 partially censored)\\
Sensitivity log correlation; median ratio & $0.9788$; $1.0055$\\
\bottomrule
\end{tabular}
\end{table}

\begin{figure}[h]
\centering
\includegraphics[width=0.98\linewidth]{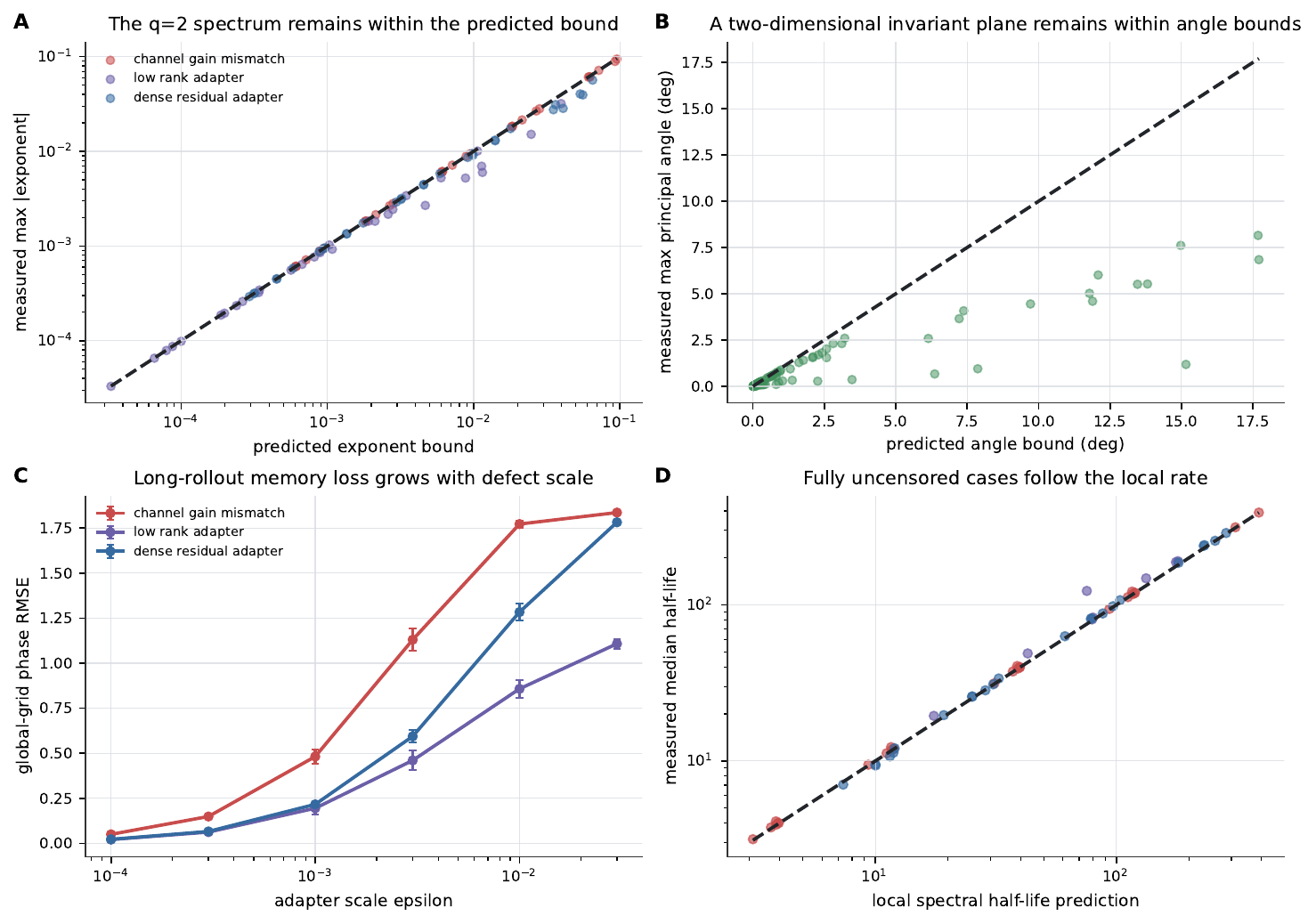}
\caption{\textbf{A $q=2$ local-to-global bridge.} (A,B) Small-defect
cases remain within the repeated-map exponent and invariant-plane angle
bounds. The 18 large-scale cases outside the effective-gap hypothesis
are absent from these two panels but retained in the released
tables. (C) Long zero-input rollouts from a global phase grid lose memory
monotonically with defect scale for every seed--adapter unit. (D) The local
two-mode spectral rate, measured after perturbation and before nonlinear
rollout, predicts median nearby pinning time for the 41 cases in which all 16
probes cross within the simulation horizon. Partially censored cases are
reported separately as a sensitivity analysis. Panels C and D are pre-rollout
bridge diagnostics, not theorem consequences for the global nonlinear map.}
\label{fig:exp40}
\end{figure}

\section{Decoder-coordinate stochastic geometry}
\label{app:stochastic_proof}

\subsection{It\^o covariance and the isotropic noise floor}
Let the ambient dynamics have local It\^o form \citep{oksendal2003}
\begin{equation}
 \dd X_t=b(X_t)\,\dd t+\Sigma(X_t)\,\dd W_t,
 \qquad Q=\Sigma\Sigma^\top.
 \label{eq:ambient_sde_app}
\end{equation}
For a $C^2$ local coordinate $\psi$, It\^o's formula gives
\[
\dd\psi(X_t)=\left[D\psi\,b+\frac{1}{2}\operatorname{Hess}(\psi):Q\right]\dd t
+D\psi\,\Sigma\,\dd W_t.
\]
The Hessian term contributes drift. The instantaneous quadratic variation is
exactly
\begin{equation}
 \dd[\psi]_t=ZQZ^\top\,\dd t,
 \qquad Z=D\psi.
 \label{eq:quadratic_variation_app}
\end{equation}

\begin{proof}[Proof of \cref{thm:decoder_floor}]
The constraint $Z\bar A=I$ implies
$Z=\bar A^\dagger+R$ with $R\bar A=0$. Since
$(\bar A^\dagger)^\top=\bar A(\bar A^\top\bar A)^{-1}$,
$R(\bar A^\dagger)^\top=0$, and therefore
\begin{equation}
 ZZ^\top=(\bar A^\top\bar A)^{-1}+RR^\top.
 \label{eq:left_inverse_decomp}
\end{equation}
Multiplication by $\sigma^2$ proves the positive-semidefinite lower bound.
The assumption $\sigma>0$ is needed for the equality characterization:
at $\sigma=0$ both covariance matrices vanish for every faithful decoder.
The inequalities then hold trivially, but equality identifies neither a unique
decoder nor an optimal action geometry.
For $\sigma>0$, equality holds exactly when $RR^\top=0$, equivalently $R=0$ and
$Z=\bar A^\dagger$.
\end{proof}

The condition $Z=\bar A^\dagger$ means the coordinate differential is the
minimum-Euclidean-norm left inverse of the action and annihilates the
orthogonal normal complement. A decoder can remain perfectly faithful along
the orbit ($Z\bar A=I$) while having additional normal sensitivity; that
sensitivity adds the positive-semidefinite term $\sigma^2RR^\top$.

\subsection{Discrete update bridge}

\begin{proof}[Proof of \cref{cor:discrete_covariance}]
Taylor expansion at $y=F(x)$ gives
\[
 \psi(y+\eta_{\Delta t})=\psi(y)+Z_y\eta_{\Delta t}+r_{\Delta t},
\]
with $\E\norm{r_{\Delta t}}^2=O(\Delta t^2)$ by assumption, whereas
$\E\norm{Z_y\eta_{\Delta t}}^2=O(\Delta t)$. Hence the covariance of the
linear term is $Z_yQ_yZ_y^\top\Delta t$, the covariance of the remainder is
$O(\Delta t^2)$, and their cross covariance is $O(\Delta t^{3/2})$ by
Cauchy--Schwarz. Centering changes only these remainder orders, proving
\cref{eq:discrete_covariance}.
\end{proof}

A bounded Hessian on the increment support together with
$\E\norm{\eta_{\Delta t}}^4=O(\Delta t^2)$ is one sufficient condition for the
remainder assumption. For Gaussian increments and a smooth decoder around a
state separated from its coordinate singularities, the same local conclusion
follows after controlling the vanishing tail outside that neighborhood.

\subsection{Anisotropic ambient noise}

\begin{proposition}[Generalized least-squares decoder]
\label{prop:gls}
Let $Q\succ0$. Among all $Z$ satisfying $Z\bar A=I$, the unique decoder minimizing
$ZQZ^\top$ in Loewner order is
\[
Z_*=(\bar A^\top Q^{-1}\bar A)^{-1}\bar A^\top Q^{-1},
\]
and the minimum covariance is
$(\bar A^\top Q^{-1}\bar A)^{-1}$.
\end{proposition}

\begin{proof}
Set $Y=ZQ^{1/2}$ and $\widetilde A=Q^{-1/2}\bar A$. The constraint becomes
$Y\widetilde A=I$. Every left inverse has
$Y=\widetilde A^\dagger+R$ with $R\widetilde A=0$, so
$YY^\top=(\widetilde A^\top\widetilde A)^{-1}+RR^\top$. Transforming back
gives the stated $Z_*$ and covariance.
\end{proof}

This is the generalized least-squares solution of
\citet{aitken1936}; the new use here is as a local recurrent-memory decoder,
not as a new linear-estimation identity.

This formula clarifies the scope of the Euclidean tight-frame result. For a
fixed action $\bar A$ and structured noise $Q$, \cref{eq:gls_decoder} optimizes
the decoder only. If the action geometry is also a design variable under the
whitened-energy constraint
$\operatorname{tr}(\bar A^\top Q^{-1}\bar A)=P$, applying the fixed-energy
argument to $Q^{-1/2}\bar A$ makes the optimum isotropic in that whitened
metric. A fixed raw Euclidean energy budget is a different problem and need
not have the same optimizer.

\subsection{Proof of the fixed-energy corollary}
\begin{proof}[Proof of Corollary~\ref{cor:tight}]
Let $\mu_1,\ldots,\mu_q>0$ be the eigenvalues of
$\bar A^\top\bar A$ with
$\sum_i\mu_i=P$. By \cref{thm:decoder_floor}, every faithful decoder satisfies
$D_\psi\succeq\sigma^2(\bar A^\top\bar A)^{-1}$. Hence
\[
\lambda_{\max}(D_\psi)\ge
\frac{\sigma^2}{\min_i\mu_i}\ge\frac{\sigma^2q}{P},
\]
and the harmonic--arithmetic mean inequality gives
\[
\tr(D_\psi)\ge\sigma^2\sum_i\mu_i^{-1}
\ge\frac{\sigma^2q^2}{P}.
\]
Equality in either scalar bound requires all $\mu_i=P/q$. Equality in the
matrix lower bound additionally requires $Z=\bar A^\dagger$. Thus both bounds are
attained exactly under the two conditions stated in the corollary.
\end{proof}

When the rows of $\bar A$ are viewed as frame vectors for the memory-coordinate
space, $\bar A^\top\bar A=(P/q)I$ is the finite tight-frame condition.
Equivalently, $\bar A$ is a scaled isometric embedding from quotient
coordinates into the state
tangent space. The inequality itself is standard frame/convex algebra; the
paper's claim is the decoder-aware recurrent-memory interpretation and causal
test \citep{benedettoFickus2003,xu2023lie,xu2025grid}.

\subsection{Angle decoder for the $T^2$ experiment}
For one circular block $z=r(\cos\theta,\sin\theta)$, the group tangent is
$a=(-z_2,z_1)^\top$ and the local angle differential is
\[
D\theta=\frac{1}{r^2}(-z_2,z_1)=a^\dagger.
\]
The $T^2$ action is free on the target orbit, so the natural generator basis
gives $\bar A=A$. For two disjoint circular blocks, $\bar A$ is block diagonal
in its two columns and the blockwise angle decoder satisfies
$D\psi=\bar A^\dagger$ exactly on the target orbit. Consequently, the
prescribed-geometry study tests the equality case of
\cref{thm:decoder_floor}, not merely its lower bound.

\section{Prescribed fixed-budget representation geometry}
\label{app:exp38}

\subsection{Prescribed representation and training protocol}
For target condition number $\kappa\ge1$ and tangent energy $P=2$, the two
circle radii are
\begin{equation}
 r_{\min}=\sqrt{\frac{2}{1+\kappa^2}},\qquad
 r_{\max}=\kappa r_{\min}.
 \label{eq:radii_condition}
\end{equation}
At any target-orbit state, the two singular values of the $T^2$ action matrix
are precisely $r_{\min}$ and $r_{\max}$. The architecture enforces independent
rotations of the two state blocks and uses only invariant features in learned
coefficients. Each condition is trained on the same two-angle velocity-input
path-integration distribution, optimizer family, update budget, and
architecture size. Seed $j$ at one geometry is paired with seed $j$ at every
other geometry.

These conditions are prescribed, scale-conjugate representations with matched
noiseless input--output dynamics by construction: the angle decoder removes
the radius scaling, while training uses the same deterministic task in each
condition. This paired design isolates the effect of isotropic ambient state
noise; it is not a sample of independently discovered representation
geometries.

The experiment pairs six seed identities across six prescribed conditions,
producing 36 evaluation cells. Summaries use the six seed identities or
condition means as specified below, not 36 independent model replicates. Geometry is prescribed by the radii. Training learns task dynamics
inside that representation and does not test spontaneous tight-frame
discovery.

\begin{table}[h]
\centering
\caption{Condition-level results at test horizon 256 and speed scale 1.5.
Uncertainty is SEM over the six paired seeds.}
\label{tab:exp38_summary}
\begin{tabular}{@{}rrrrr@{}}
\toprule
$\kappa(\bar A)$ & Predicted $\lambda_{\max}$ & Long-rollout $\lambda_{\max}$ & Noisy RMSE & Clean RMSE\\
\midrule
1.0 & $0.000900$ & $0.000920\pm0.000014$ & $0.204\pm0.014$ & $0.135\pm0.021$\\
1.5 & $0.001462$ & $0.001481\pm0.000015$ & $0.225\pm0.013$ & $0.134\pm0.020$\\
2.0 & $0.002250$ & $0.002337\pm0.000047$ & $0.266\pm0.012$ & $0.136\pm0.021$\\
3.0 & $0.004500$ & $0.004490\pm0.000058$ & $0.353\pm0.010$ & $0.136\pm0.022$\\
5.0 & $0.011700$ & $0.011772\pm0.000126$ & $0.543\pm0.007$ & $0.136\pm0.020$\\
8.0 & $0.029250$ & $0.030317\pm0.000489$ & $0.843\pm0.010$ & $0.134\pm0.019$\\
\bottomrule
\end{tabular}
\end{table}

The maximum absolute tangent-energy deviation from two is
$2.519\times10^{-7}$. Maximum double-precision step-equivariance and
group-tangent covariance residuals are $1.999\times10^{-16}$ and
$2.483\times10^{-16}$. The median relative errors of the full local and
long-rollout diffusion tensors are $0.560\%$ and $3.317\%$. Pearson
correlation between predicted and long-rollout worst diffusion is $0.998912$
across all 36 cells and $0.999920$ across the six condition means; condition-
mean Spearman correlation is one.

Every one of the six seed identities has strictly increasing noisy
horizon-256 worst-coordinate RMSE. The paired endpoint difference
$\mathrm{RMSE}_{\kappa=8}-\mathrm{RMSE}_{\kappa=1}$ is
$0.638943\pm0.010973$ SEM and is positive for all six seeds. In contrast, the
range of clean condition means is $0.002875$, the seed-fixed clean slope is
statistically indistinguishable from zero in the regression specified before analysis, and the
maximum validation-loss range across geometries within a seed is
$3.492\times10^{-10}$. Decoder sensitivity grows from $1.0$ to $5.70$, as
expected from the shrinking radius of the weakest coordinate.

Noise is added after each recurrent update as
$\eta_k\sim\mathcal N(0,\sigma^2\Delta t I)$ with $\Delta t=0.1$. Therefore
\cref{cor:discrete_covariance} is the formal one-step, small-$\Delta t$ bridge.
The reported finite-step tensor fit and horizon-256 task curves additionally
test repeated nonlinear composition and should not be read as an exact
continuous-time covariance identity at $\Delta t=0.1$.

\begin{figure}[h]
\centering
\includegraphics[width=0.95\linewidth]{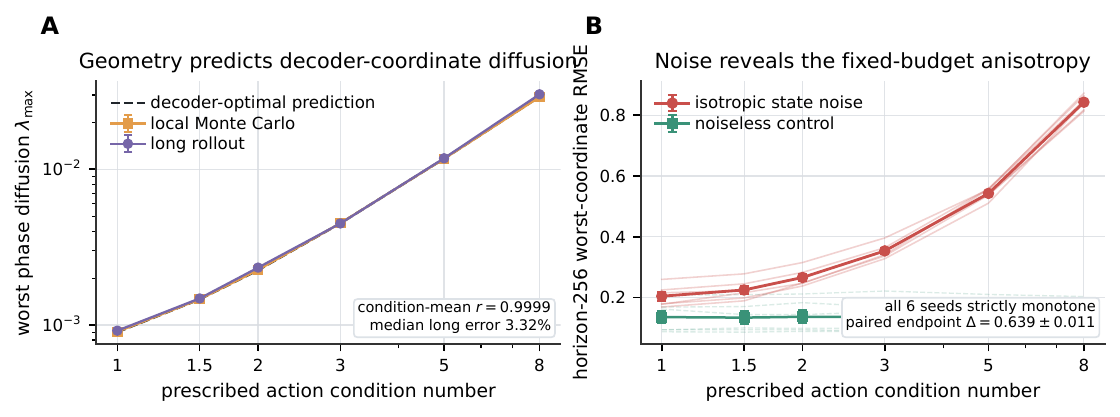}
\caption{\textbf{Prescribed-geometry calibration underlying the
learned-geometry study.}
The least-squares prediction tracks local and long-rollout worst-coordinate
diffusion across six prescribed condition numbers (left). Thin lines pair the
same six seed identities across conditions (right): every noisy curve worsens
with anisotropy while the scale-conjugate noiseless control remains nearly
flat. Geometry is prescribed here; learned selection is evaluated in the fresh
replication in \cref{app:fresh_replication}.}
\label{fig:exp38_full}
\end{figure}

\section{Theory-guided learning of action geometry}
\label{app:exp41}

\subsection{Exact parameterization and arms}
The learned-geometry study retains the prescribed-geometry recurrent cell,
blockwise angle decoder, two-angle velocity task, $P=2$ budget,
$\sigma=0.03$ ambient noise,
$\Delta t=0.1$, 180-update training budget, and horizon-256 evaluation. Only
the orbit radii become trainable. Two logits enforce
\begin{equation}
 w=\operatorname{softmax}(\alpha),\qquad
 r_i^2=2w_i,\qquad
 r_1^2+r_2^2=2
 \label{eq:learned_radii}
\end{equation}
identically up to floating-point roundoff. The group action remains independent
rotation of the two state blocks, and every learned coefficient depends only
on invariant features. Thus geometry learning preserves exact $T^2$
equivariance rather than approximating it through augmentation.

Five paired arms use the same recurrent initialization and minibatch stream
for a given seed: fixed $\kappa=8$, learnable geometry with clean loss,
learnable geometry with state noise, clean loss plus the normalized covariance
penalty
\begin{equation}
 \mathcal L_{\mathrm{geom}}
 =\tfrac12\tr\!\left[(\bar A^\top\bar A)^{-1}\right]-1,
 \label{eq:geometry_regularizer}
\end{equation}
and a fixed $\kappa=1$ oracle. All arms contain the same total number of
parameters; fixed arms simply disable gradients for the two geometry logits.
The clean scale-conjugacy test verifies that changing fixed radii does not add
deterministic task capacity.

\subsection{Hyperparameter selection and evaluation cohorts}
\label{app:geometry_selection}
The covariance-regularization coefficient was selected using a clean-only
pilot with seeds 0 and 1, 120 updates, and candidate coefficients
$\{10^{-3},3\times10^{-3},10^{-2}\}$. The initial conditioning target
$\kappa\le1.5$ was not reached and was revised to $\kappa\le2.5$ before
evaluating the original six-unit cohort. The smallest coefficient meeting
the revised target while preserving matched clean validation performance,
$10^{-3}$, was selected. No noisy test metric entered this decision.
Additional development evaluations at coefficient $3\times10^{-3}$ with
seed labels 0--5 were excluded from the original evaluation summaries.

The original evaluation used labels 6--11, which were not used for model or
hyperparameter selection. Subsequent inspection found that its arithmetic
generator-seed formulas could reuse random streams across training roles
and adjacent units. Distinct displayed labels therefore did not establish
independent random streams. We retain these observations as descriptive
historical records, do not interpret their nominal SEMs as uncertainty from
independent training replicates, and do not pool them with the fresh cohort.
The full pilot details, threshold amendment, excluded development records,
six-unit results, and structural diagnostics are retained in the
supplementary artifact; \path{historical_t2/README.md} identifies the
archived records and their interpretation.

Our principal $T^2$ learning evidence is the fresh twenty-unit replication
in \cref{app:fresh_replication}. Its protocol was specified after inspecting
the original results and before evaluating the new cohort, with the selected
coefficient and training settings held fixed. The revised generator schedule
separates random roles and training units while preserving paired
initialization and data streams across arms within each unit. The
replication's outcomes, uncertainty estimates, and clean-error tradeoffs
are reported in \cref{app:fresh_replication}.

\section{Fresh replication and retrospective bound usefulness}
\label{app:fresh_replication}

\subsection{Fixed protocol and independent training units}
After inspecting the original results, we specified the replication protocol
and recorded its code hashes before evaluating the new cohort; the protocol
was not externally preregistered. Twenty training units, labeled $10000,10100,\ldots,11900$,
each trained all five arms without selecting or replacing seeds. The
architecture, loss, coefficient $10^{-3}$, 180 updates, batch size 96, training
horizon 64, invariant hidden size 8, feature width 24, learning rates
$0.003$ for recurrent parameters and $0.02$ for geometry, $\Delta t=0.1$,
$P=2$, and $\sigma=0.03$ were held fixed. The same initialization and data
streams were shared across arms within each unit, not across units.
All 100 confirmatory runs completed with finite recorded metrics.

To prevent generator-seed reuse, overlapping arithmetic seed formulas were
replaced by explicit SHA-256-derived namespaces for study, dimension, stage,
training unit, random role, and evaluation index. A collision check verified every scheduled
generator seed. Training data, training noise, validation data, test data,
test noise, and bootstrap draws have distinct roles; pilot identities are
separate. An implementation regression mapped the substituted seeds back to
the archived values and recovered identical training curves and weights.
This is a same-distribution replication with corrected stream separation,
not a byte-identical replay of the original experiment.

The noisy task uses 1,000 trajectories per model at horizon 256 and speed
scale 1.5. For each coordinate, circular RMSE is computed across the endpoint errors of
the evaluation trajectories; the reported metric is the maximum of these
coordinatewise RMSEs.
Local covariance uses 40,000 increments per model; long covariance uses
1,500 zero-input trajectories over 300 updates, with covariance traces and
raw increments retained. All metrics were independently reconstructed from
those arrays. Trajectories are nested samples, not additional training seeds.
The primary contrast was covariance-regularized minus clean-only noisy RMSE;
all remaining contrasts are descriptive secondary analyses.

\subsection{Results and clean-error equivalence}
\begin{table}[t]
\centering
\caption{Fresh fixed-protocol replication: twenty paired training units per
arm. Values are mean $\pm$ SEM across training units. RMSE is in radians;
$\kappa$ is the action-map singular-value condition number.}
\label{tab:fresh_replication}
\begin{tabular}{@{}lrrr@{}}
\toprule
Arm & Final $\kappa$ & Noisy RMSE & Clean RMSE\\
\midrule
Fixed anisotropic & $8.000\pm0.000$ & $0.845\pm0.003$ & $0.103\pm0.005$\\
Learnable clean & $8.065\pm0.016$ & $0.852\pm0.004$ & $0.103\pm0.005$\\
Learnable noisy & $1.757\pm0.005$ & $0.232\pm0.003$ & $0.121\pm0.006$\\
Covariance regularized & $1.755\pm0.000$ & $0.229\pm0.002$ & $0.103\pm0.005$\\
Fixed isotropic oracle & $1.000\pm0.000$ & $0.179\pm0.003$ & $0.103\pm0.005$\\
\bottomrule
\end{tabular}
\end{table}

The primary paired mean is $-0.622977$ rad, with SD $0.015187$, SEM
$0.003396$, a two-sided 95\% $t_{19}$ interval $[-0.630085,-0.615869]$,
and a 10,000-resample seed-bootstrap interval $[-0.629221,-0.616291]$.
All twenty differences are negative. The covariance arm's paired clean
change is $-4.36\times10^{-8}$ rad, with interval
$[-1.28\times10^{-7},4.09\times10^{-8}]$, contained in the prespecified
$[-0.01,0.01]$ rad equivalence margin. This establishes equivalence under
that operational margin, not mathematical equality of trained predictors.

Direct noise training has paired noisy change $-0.620221$ rad but clean
change $+0.018431$ rad, with 95\% interval $[0.006266,0.030595]$.
It does not meet the clean-equivalence criterion. Thus the original
six-unit clean-error observation does not extend to this larger cohort for
the noise-trained arm. No adverse run is discarded. Covariance regularization
still trails the fixed isotropic oracle in noisy RMSE by $0.050285\pm0.001434$
rad. The original six-unit observations and the descriptive cohort comparison
remain available separately in the supplementary artifact; they are not
pooled with this replication. Full SD, SEM, paired intervals, all seed values,
and failure-sensitive analyses are provided in the artifact.

\subsection{Retrospective usefulness of resolving defect blocks}
\label{app:bound_usefulness}
We reconstruct 4,566 archived maps or finite sequences in their original
reducing frames, verifying their four block norms against the saved tables.
In those same coordinates we compute the true full operator norm
$\delta=\|E\|_2$, rather than substituting the intervention amplitude or a
$2\times2$ norm envelope. Since every block norm is at most $\delta$, a
structure-blind comparator can validly pass $a=b=c=e=\delta$ to the same
sufficient-condition algorithm. Both methods use identical normal
contraction, durations, endpoint conditioning, and specified initial graph.
This is one norm-only comparator, not the best possible norm-only theorem.

\begin{table}[h]
\centering
\caption{Retrospective bound availability, with all archived cases retained.
The denominator is the complete case set, including failed sufficient
conditions. These nested cases are not independent training seeds. Neither
method has a measured coverage violation among its valid cases.}
\label{tab:bound_usefulness}
\begin{tabular}{@{}lrrr@{}}
\toprule
Archived study & Cases & Block-resolved valid & Norm-only valid\\
\midrule
Designed routing & 1,608 & 1,584 & 1,404\\
Ambient residuals & 2,520 & 2,515 & 2,160\\
Native residual coordinates & 90 & 90 & 90\\
Learned $T^2$ adapters & 108 & 90 & 72\\
Finite-horizon sequences & 240 & 216 & 216\\
\bottomrule
\end{tabular}
\end{table}

\begin{figure}[t]
\centering
\includegraphics[width=\linewidth]{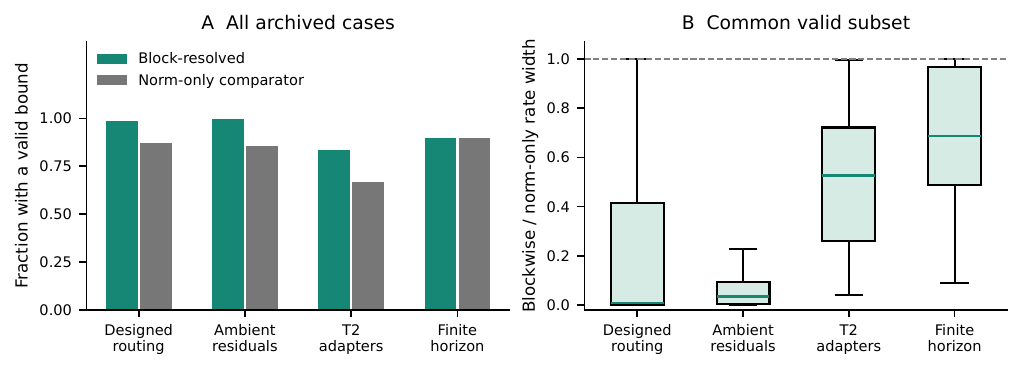}
\caption{\textbf{Retrospective usefulness, not prospective forecasting.}
(A) Sufficient-bound availability over all cases in four archived studies.
(B) Rate-interval width ratio on the common valid subset; boxes show quartiles
and whiskers 1.5 interquartile ranges, with outlier markers omitted only from
the visualization. All cases remain in the tables. The native-coordinate
control is reported in \cref{tab:bound_usefulness}. Availability improves in
three displayed local studies but not for the finite-horizon sequences;
sharper widths and higher availability are distinct benefits.}
\label{fig:bound_usefulness}
\end{figure}

Invalid bounds remain unavailable, never zero. Sharpness ratios are computed
with all zero-defect controls retained: the 1,584 valid designed-routing cases
include 12 zero-defect cases, whereas 1,572 counts only nonzero defects.
Ratios are computed
only on the common valid subset and above the original numerical floors and
duration-adjusted coverage tolerance. The finite-horizon comparison retains
all 24 out-of-regime controls and does not show an availability advantage. Analytic
three-dimensional controls with identical block norms give positive, negative,
or zero center displacement according to leakage orientation; a neutral
control can also fail the sufficient criterion. Norm information therefore
does not identify the sign of displacement, and failure of the sufficient
criterion is not evidence of instability. This retrospective analysis neither
extends the theorem's scope nor provides an online guarantee.

\subsection{Noise-dependent geometry in coupled torus integrators}
\label{app:noise_geometry}
\label{app:geometry_stop}

\paragraph{Fixed-Euclidean-budget specialization.}
For orthogonal charge-one planes, write $p_i=r_i^2>0$ and fix
$\sum_i p_i=P>0$. The decoder is the local blockwise angle map, not a learned
or freely optimized decoder. For fixed ambient block covariance
$Q_i=\sigma^2v_iI_2$, with $v_i>0$ and $\sigma>0$, its differential has norm
$1/r_i$ on plane $i$, so \cref{eq:decoder_covariance} gives
\[
 D_\psi=\diag(\sigma^2v_i/p_i),\qquad
 \tr D_\psi=\sigma^2\sum_i v_i/p_i.
\]
Cauchy--Schwarz applied to $\sqrt{v_i/p_i}$ and $\sqrt{p_i}$ gives
\[
 \left(\sum_i\sqrt{v_i}\right)^2
 \leq \left(\sum_i v_i/p_i\right)\left(\sum_i p_i\right).
\]
Equality holds exactly when $p_i$ is proportional to $\sqrt{v_i}$.
Consequently the unique positive trace-minimizing allocation is
\[
 p_i^*=\frac{P\sqrt{v_i}}{\sum_j\sqrt{v_j}},\qquad
 R_v(p)=\frac{P}{(\sum_j\sqrt{v_j})^2}\sum_i\frac{v_i}{p_i}-1\geq0.
\]
It is squared radius, not radius, that scales as the square root of the
variance multiplier; radius scales as $v_i^{1/4}$. For the different objective
$\min\max_i\sigma^2v_i/p_i$, if $M=\max_i v_i/p_i$, then
$P\geq\sum_i v_i/M$, with equality exactly at
$p_i=P v_i/\sum_jv_j$. Equal $v_i$ recover $p_i=P/q$; unequal $v_i$ generally
select anisotropic Euclidean geometry. This elementary specialization does
not optimize arbitrary full covariance, a global integrable generalized
least-squares decoder, or long-time nonlinear task loss. The Euclidean budget
here differs from a noise-whitened energy constraint. Invariant hidden-state
noise has zero instantaneous derivative under angle decoding but can affect
later recurrent dynamics. Numerical checks of $R_v$ and its logit gradient
agree with the closed form and central differences. Formalization scope is
summarized in \cref{app:formal_correspondence}.

\paragraph{Coupled equivariant architecture.}
The state consists of $q$ charge-one planes $z_i\in\R^2$ and invariant
$h\in\R^{32}$; $T^q$ rotates each $z_i$ independently. Squared target radii
are $p_i=q\,\mathrm{softmax}(\alpha)_i$, and
$e_i=\|z_i\|^2/p_i-1$. A shared 64-wide tanh network maps the invariants
$(e,h,u,u^2)$ to scalar coefficients $c_i$, angular coefficients $b_i$, and
hidden coefficients $g$. With $K=I+0.2\mathbf{1}\mathbf{1}^{\top}/q$ and
learned positive gains $a_i,\omega_i,\gamma$, the field is
\begin{align*}
 \dot z_i&=\{-a_i(Ke)_i+0.04e_i\tanh(c_i)\}z_i
          +u_i(\omega_i+0.04b_i)Jz_i,\\
 \dot h&=-\gamma h+0.04\Big(\sum_i e_i+\sum_i u_i^2\Big)g.
\end{align*}
Every coefficient is invariant, so the field and Euler update are exactly
equivariant in real arithmetic. The zero-input target torus with $h=0$ is
fixed and has $q$ independent group tangents; it has no additional nonzero
flow vector. The vector-output loss uses $z_i/(\|z_i\|+10^{-8})$; this
readout is not a normalization of the hidden dynamics.

\paragraph{Pilot development.}
The first unbounded-radial implementation had two clean and two noisy
engineering pilots. The noisy $T^4/T^8$ pilots failed with nonfinite
objectives at updates 160/188. These failures stopped confirmation; their
records remain archived. A single documented amendment, termed BR1 in the
artifact, replaced $0.04e_ic_i$ by $0.04e_i\tanh(c_i)$, with no other executable
model change. The scalar correction is bounded by $0.04|e_i|$ and its vector
contribution by $0.04|e_i|\|z_i\|$; neither is a global bound on states or a
stability theorem for Euler under unbounded Gaussian noise. The same four
pilot identities were replayed, not replaced by a new independent cohort.
All four met the original engineering criteria. The amended model and
evaluation protocol were then fixed before confirmation; this was not
external preregistration. The initial failures and same-identity replays
remain available in the supplementary records.

\paragraph{Protocol and pairing.}
There are 192 confirmation runs: six arms times sixteen independent paired
training units in each of $q=4,8$. The four pilots are separate. Arms are fixed
anisotropic geometry, clean geometry learning, noise-trained geometry,
anisotropic-covariance regularization, isotropic regularization, and a fixed
analytical trace oracle. The subtraction ``anisotropic-covariance minus
isotropic regularization'' is held fixed when switching evaluation noise;
the labels do not imply that one arm matches every noise condition.
Initial nongeometry weights and role-specific random streams are paired
across arms; SHA-256-derived actual generator seeds separate dimensions,
stages, units, and roles. There is no pooling with the earlier $T^2$ cohorts.

All runs use 360 AdamW updates, horizon 64, batch 96, hidden size 32,
feature width 64, $\Delta t=0.1$, gradient clipping at 1, base learning rate
0.003 and geometry learning rate 0.02. Base weight decay is $10^{-5}$;
geometry weight decay is zero. The regularizer weight is 0.001. Clean
vector-output MSE trains all arms except noise training, which injects the
declared state noise. Final weights at update 360 are evaluated without
best-checkpoint selection. Total parameter counts are 5,493 for $T^4$ and
6,793 for $T^8$; the fixed arms freeze $q$ geometry parameters. Initial action
condition number is 8. Gaussian, piecewise-constant, and correlated velocities
use training speed scale 0.8 and uniform initial phases.

Noise has $\sigma=0.03$ and a seed-permuted geometric profile from 1 to 64,
normalized to mean one. The ambient covariance is fixed independently of
learned radii: $\sigma^2\diag(v_1I_2,\ldots,v_qI_2,I_{32})$.
Isotropic evaluation replaces $v$ by ones, with unchanged total covariance
trace. The energy budget $P=q$ is fixed within each dimension; it holds
per-coordinate energy, not total energy across dimensions, constant.
The primary endpoint is mean squared wrapped endpoint phase error over
1,000 trajectories and $q$ coordinates at horizon 256, physical time 25.6,
and speed 1.5. It is MSE in rad$^2$, not time-averaged RMSE or
worst-coordinate RMSE. Horizons 64/128 and all other contrasts are secondary.
Inference uses sixteen paired training units, not trajectories as replicates.
Both paired $t$ intervals and 10,000 paired-unit percentile-bootstrap
intervals are reported, without multiplicity-controlled claims for secondary
comparisons. No failed or unfavorable confirmation run was dropped.

\begin{table}[t]
\centering
\caption{Absolute endpoint phase MSE in rad$^2$ at horizon 256, speed 1.5.
Mean $\pm$ SEM across sixteen training units per arm and dimension; clean,
anisotropic-noise, and isotropic-noise evaluations use the same final cells.}
\label{tab:br1_absolute}
\small\setlength{\tabcolsep}{3pt}
\begin{tabular}{@{}rlrrr@{}}
\toprule
$q$ & Arm & Clean & Anisotropic & Isotropic\\
\midrule
4 & Fixed anisotropic & $0.002382\pm0.000031$ & $0.134828\pm0.030448$ & $0.157644\pm0.001403$\\
4 & Clean learning & $0.002382\pm0.000031$ & $0.134851\pm0.030542$ & $0.157763\pm0.001327$\\
4 & Noise training & $0.002724\pm0.000069$ & $0.017314\pm0.000467$ & $0.038315\pm0.001035$\\
4 & Anisotropic-cov. loss & $0.002382\pm0.000031$ & $0.017494\pm0.000461$ & $0.038390\pm0.001053$\\
4 & Isotropic loss & $0.002382\pm0.000031$ & $0.023600\pm0.000279$ & $0.024161\pm0.000101$\\
4 & Trace oracle & $0.002382\pm0.000031$ & $0.016646\pm0.000089$ & $0.040753\pm0.000231$\\
\midrule
8 & Fixed anisotropic & $0.002561\pm0.000032$ & $0.106317\pm0.014793$ & $0.109176\pm0.000602$\\
8 & Clean learning & $0.002561\pm0.000032$ & $0.106292\pm0.014806$ & $0.109130\pm0.000572$\\
8 & Noise training & $0.002822\pm0.000040$ & $0.017743\pm0.000124$ & $0.033753\pm0.000559$\\
8 & Anisotropic-cov. loss & $0.002561\pm0.000032$ & $0.018047\pm0.000123$ & $0.033974\pm0.000584$\\
8 & Isotropic loss & $0.002561\pm0.000032$ & $0.024234\pm0.000165$ & $0.024160\pm0.000106$\\
8 & Trace oracle & $0.002561\pm0.000032$ & $0.017767\pm0.000091$ & $0.036163\pm0.000151$\\
\bottomrule
\end{tabular}

\end{table}

\begin{table}[t]
\centering
\caption{Paired endpoint MSE differences in rad$^2$. A means anisotropic
evaluation noise and I isotropic noise; the covariance-minus-isotropic arm
subtraction never changes. Only the starred row is primary. Other rows are
descriptive and the two interval procedures are retained even when they differ.}
\label{tab:br1_contrasts}
\small\setlength{\tabcolsep}{3pt}
\begin{tabular}{@{}rllrrr@{}}
\toprule
$q$ & Contrast & Noise & Mean & 95\% $t$ interval & 95\% bootstrap\\
\midrule
4 & Cov.--isotropic$^*$ & A & $-0.006106$ & $[-0.006674,-0.005538]$ & $[-0.006567,-0.005563]$\\
8 & Cov.--isotropic & A & $-0.006188$ & $[-0.006335,-0.006040]$ & $[-0.006323,-0.006061]$\\
4 & Cov.--isotropic & I & $+0.014229$ & $[+0.011963,+0.016496]$ & $[+0.011913,+0.015956]$\\
8 & Cov.--isotropic & I & $+0.009814$ & $[+0.008736,+0.010891]$ & $[+0.008847,+0.010767]$\\
4 & Cov.--oracle & A & $+0.000848$ & $[-0.000074,+0.001771]$ & $[+0.000076,+0.001747]$\\
8 & Cov.--oracle & A & $+0.000280$ & $[+0.000091,+0.000469]$ & $[+0.000125,+0.000460]$\\
\bottomrule
\end{tabular}

\end{table}

\paragraph{Task results and control comparisons.}
\Cref{tab:br1_absolute,tab:br1_contrasts} give absolute means and paired
differences. The $T^4$ primary advantage and corresponding $T^8$ secondary
advantage reverse under isotropic evaluation noise. All 36 declared clean
per-coordinate contrast intervals lie inside $\pm0.01$ rad, a conditional
equivalence criterion for these contrasts, not evidence that every training
objective is cost-free. In particular the earlier $T^2$ noise-training cost
$+0.018431$ rad, interval $[0.006266,0.030595]$, remains an adverse result.
The $T^4$ covariance-minus-oracle anisotropic-noise difference is
$+0.000848168$ rad$^2$: its $t$ interval crosses zero while its bootstrap
interval does not. Neither demonstrates equivalence. The $T^8$ difference
$+0.000280051$ has both intervals above zero. The local trace oracle is not
claimed optimal for this nonlinear finite-horizon task.

Geometry swaps change only the final clean-trained cell's geometry logits;
nongeometry weights, inputs, and noise streams are identical within each
comparison. Swapping to trace-optimal geometry lowers anisotropic-noise MSE
in all sixteen units per dimension. Swapping to isotropic geometry improves
the mean but worsens one $T^4$ unit under anisotropic noise (15/16 improve).
Clean MSE changes are numerical nulls: absolute paired means are at most
$2.1\times10^{-10}$ rad$^2$. Deterministic phase conjugacy is checked
separately; the $10^{-8}$ output-normalization offset permits a small
radius-dependent vector magnitude difference, so vector readouts are not
claimed algebraically identical. All swap pairs and both interval estimates
are included, not just favorable mean effects.

\paragraph{Local covariance and finite-step sensitivity.}
Local covariance uses 40,000 increments at each of three phase points with
$\Delta t=5\times10^{-5}$; full centered covariance matrices are divided by
physical time. Long zero-input covariance uses 1,500 trajectories, 300 steps
at $\Delta t=0.1$, unwrapped phase, and a full-matrix covariance-versus-time
linear fit. The local decoder prediction is not itself a long-time theorem.
Across cells, three-phase-averaged local relative Frobenius error ranges from
0.27 to 1.50 percent; long-rollout error ranges from 1.14 to 12.43 percent.
No range is truncated in \cref{fig:br1_calibration}.
Confirmation long rollouts record 7,325 branch-wrap events, two phase jumps
larger than $\pi/2$, no radii below $10^{-3}$, and no nonfinite state values.
These observed event counts are not a proof of global stochastic stability.

\begin{figure}[t]
\centering
\includegraphics[width=\linewidth]{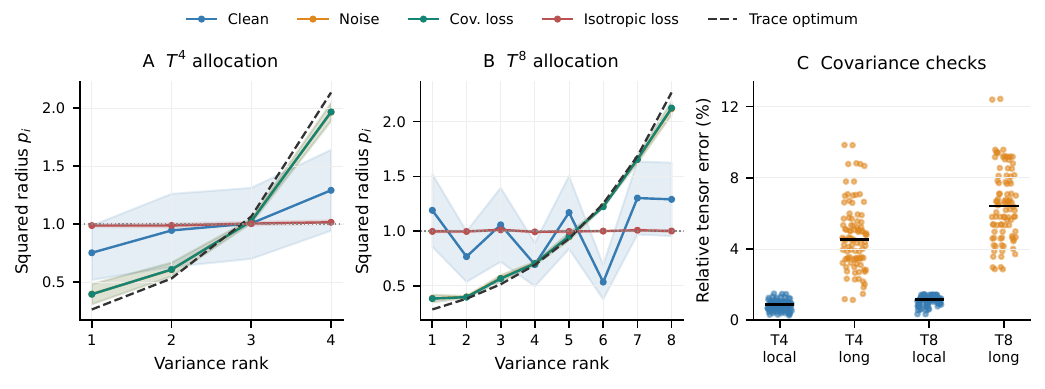}
\caption{\textbf{Noise allocation and covariance calibration.}
(A/B) Learned squared radii versus the variance-rank coordinate, aligned
within each seed before averaging; shaded bands are SEM across sixteen
units, with the analytic trace allocation and isotropic allocation shown.
(C) Full relative Frobenius-error ranges for local and long covariance
estimates; dots are all 192 confirmation cells, not independent covariance
replicates. Finite-step long-rollout agreement is a conditional check, not
an asymptotic stochastic theorem.}
\label{fig:br1_calibration}
\end{figure}

Six fixed cases were also evaluated at $\Delta t=0.05$ for 600 steps, keeping
physical duration 30. \Cref{tab:br1_stepsize} retains all six outcomes:
tensor error is larger in each half-step evaluation. These runs use separate
noise streams, not coupled Brownian refinements; two step sizes therefore do
not establish discretization convergence. Relative trace changes range from
$-1.583\%$ to $+1.606\%$.
\begin{table}[h]
\centering
\caption{All fixed step-size sensitivity cases. Errors are relative
Frobenius percentages; trace change is half-step minus nominal, divided by
nominal. Each row uses the first confirmation unit and independent noise
streams at the two step sizes.}
\label{tab:br1_stepsize}
\begin{tabular}{@{}rlrrr@{}}
\toprule
$q$ & Arm & Error at $0.1$ & Error at $0.05$ & Trace change (\%)\\
\midrule
4 & Fixed anisotropic & $4.631$ & $5.665$ & $-0.602$\\
4 & Trace oracle & $5.083$ & $6.680$ & $-1.583$\\
4 & Anisotropic-cov. loss & $5.083$ & $6.680$ & $-1.583$\\
8 & Fixed anisotropic & $5.802$ & $7.558$ & $+1.606$\\
8 & Trace oracle & $6.680$ & $8.766$ & $-0.549$\\
8 & Anisotropic-cov. loss & $6.685$ & $8.761$ & $-0.502$\\
\bottomrule
\end{tabular}

\end{table}

\paragraph{Structural checks and corrected operator norms.}
Final-checkpoint loading, parameter counts, invariant covariance construction,
and stored normal-block reconstruction are rechecked for all 196 cells.
Maximum confirmation float64 step-equivariance and direct tangent-transport
errors are below $2.34\times10^{-16}$ and $3.15\times10^{-16}$, respectively; action singular
values remain positive in the sampled checks. These are real-arithmetic
architecture identities tested numerically, not finite-precision exactness.

An installed-runtime discrepancy affected the original descriptive normal
operator-norm field in 95 of 196 records (92 confirmation, three pilots).
The saved matrices are unchanged. NumPy SVD, the square root of the largest
Gram eigenvalue, and the explicit maximum of PyTorch singular values agree.
A minimal example has original value 0.6650574 but corrected norm 1.0795389;
the standalone fixture records shape, strides, dtype, software, and exact
call. The cause is unresolved, and no claim about other installations is made.
The canonical correction table is joined by stage, dimension, unit, and arm;
raw scalar reports are retained for transparency. Corrected operator norms
are at least one in 153/196 cells, including 149/192 confirmation cells.

All recomputed normal spectral radii remain below one, as required by the
original engineering check. Spectral radius below one is not one-step
Euclidean contraction, nor does it establish contraction of arbitrary
time-varying products or global nonlinear noisy stability. Therefore these
checks do not instantiate the normal-contraction hypotheses of the
deterministic perturbation propositions. Source-use tracing confirms that
this descriptive field did not affect training, final-weight selection,
primary task endpoints, or covariance estimates. The separate deterministic
studies use NumPy operator norms and retain their own normal-contraction
requirements. These stochastic results do not extend the deterministic
propositions or the Lean formalization.

\section{Additional related work}
\label{app:related_extended}
Recent work studies learned manifolds, approximate persistence, symmetry
regularization, and infinite-horizon approximation of normally hyperbolic
continuous attractors
\citep{darshanRivkind2022,sagodi2024back,liang2025symreg,sagodiPark2026uap}.
Connectivity roughness can induce drift toward discrete minima, while
constrained optimization can stabilize multiple embedded attractor manifolds
\citep{agmon2023multiple}. Robust memory manifolds can also arise without exact
symmetry \citep{can2025robust}.
Cotteret et al. analyze a complementary stability--resolution tradeoff under
noise and heterogeneity for distributed grid-cell-like codes; our stochastic
result instead conditions decoded covariance on a fixed action, decoder,
metric, and energy budget \citep{cotteret2025spatial}.
Representation-theoretic treatments broaden equivariant architectures
\citep{kondorTrivedi2018}, while approximate variants relax hard constraints
\citep{wang2022approximately}. Our bounds instead condition local memory
degradation on four measured tangent/normal blocks and give no global
guarantee from a scalar defect.
Recurrent Equivariant Constraint Modulation learns layerwise relaxation from
data-dependent symmetry gaps \citep{pertigkiozoglou2026recm}. Our setting is
complementary: exact per-input equivariance protects the reference memory
bundle, while explicit off-equivariant defects are diagnosed through their
tangent/normal blocks rather than represented by learned relaxation levels.
Recent hippocampal RNN work reports a learned mixed-symmetry regime in which
dominant symmetric recurrence stabilizes an attractor and a weaker
antisymmetric component generates directed flow \citep{wagner2026balancing}.
This is complementary to our setting: Wagner et al. study learned motion along
an attractor, whereas our bounds condition degradation of a protected memory
bundle on measured tangent/normal perturbation blocks.

Certified equivariant world models bound different objects: rollout error
across symmetry configurations, Lyapunov-limited horizons, and conformally
calibrated orbit-valid trust regions \citep{wang2026certified,wang2026conformal}.
Our analysis concerns continuation and rotation of the protected
group-tangent bundle itself. Numerical Lyapunov algorithms
\citep{oseledets1968,benettin1980Theory,benettin1980Numerical} and covariant
subspaces \citep{ginelli2007} diagnose asymptotic growth; Jacobian
regularization tunes propagation directly \citep{engelken2023gradientflossing}.
The present block decomposition instead identifies where a measured defect acts.
Classical invariant-subspace perturbation theory controls subspace rotation
using spectral separation or residual size
\citep{davisKahan1970,stewart1973subspaces}. Our result is complementary: it
resolves four tangent/normal routes for a generally nonnormal recurrent map and
converts them into memory-rate bounds.

Noise projection near cycles and tori
\citep{maclaurin2023phase,ahsan2025covariance} and drift--diffusion matching
\citep{nartallo2026driftdiffusion} are complementary forward problems. We ask
the inverse local design question: for a fixed action, decoder, and ambient
noise, which decoded covariance is unavoidable? The answer is conditional on
those choices and does not assert a universal representation optimum.

Input-driven continuous-attractor RNNs also use invariant manifolds for
representation transfer and domain adaptation \citep{xu2026transfer}. Their
goal differs from ours: we ask when exact per-input equivariance transports
recurrent memory tangents and how measured defects and state noise act on the
decoded coordinates. Temporal-scaling equivariant circuits encode a group
parameter through a control input to generate sequences at different speeds
\citep{zuo2023temporal}; our controlled-cocycle result instead concerns pathwise
tangent transport and finite-horizon degradation.

Formal verification of neural networks is developing rapidly. TorchLean gives
neural-network definitions, execution, finite-precision semantics, and
verification a shared Lean semantics \citep{george2026torchlean}. Our
development instead machine-checks selected dynamical bounds and
architecture-level equivariance, while correspondence to the released PyTorch
checkpoints is tested executably rather than formalized as PyTorch semantics.
Formalized Hopfield and Boltzmann models provide a complementary Lean treatment
of classical recurrent-memory convergence and stochastic ergodicity
\citep{cipollina2025hopfield}.

\section{Learned-task and baseline scope checks}
\label{app:task_scope}
The earlier $S^1$ learned experiment grounds the framework in a recognizable
continuous-memory task while remaining application evidence rather than
theorem proof. The exact cell has state $(z,h)$ with
$z\in\R^2$ transforming by rotation and $h$ invariant. Its vector field has
form
\[
\dot z=a(\norm z^2,h,u)z+b(\norm z^2,h,u)Jz,
\qquad
\dot h=g(\norm z^2,h,u),
\]
which is equivariant for each scalar invariant control $u$. The zero-input
restriction is covered subject to the autonomous theorem's hypotheses, and the per-input theorem in
\cref{thm:controlled_neutral} now also gives exact tangent transport for
arbitrary input sequences as long as the orbit rank remains fixed. Task error
is still consequence evidence, because good decoding and transverse stability
are additional requirements.

The matched scope checks use GRU \citep{cho2014gru}, LSTM
\citep{hochreiter1997}, and orthogonal or unitary recurrent baselines
\citep{henaff2016,arjovsky2016}. They provide task-level context under the
reported budgets rather than a claim that generic gated or norm-preserving
recurrent networks cannot integrate.

The archived comparison uses training horizon 64, batch size 64, AdamW and
gradient clipping at 1. The generic-baseline reruns use learning rate 0.002;
the inherited equivariant reference uses 0.003. Thus these are controlled
budget comparisons, not identical-hyperparameter retraining of all models.
The per-seed records, configurations and table-generation command are included
under \path{empirical_provenance/baseline_scope/}.

\begin{table}[h]
\centering
\caption{Original and stronger generic-baseline check on the horizon-256,
speed-1.8 full-phase slice. The stronger baselines use hidden size 32 and 500
optimization steps. Entries are mean $\pm$ SEM over three seeds per baseline
configuration and six seeds for the equivariant reference.}
\label{tab:baseline_scope}
\begin{tabular}{@{}lcc@{}}
\toprule
Configuration & Full-phase circular RMSE & Restricted-phase RMSE\\
\midrule
Equivariant reference & $0.2554\pm0.016$ & --\\
GRU, original budget & $1.748\pm0.023$ & $1.778\pm0.008$\\
GRU, stronger budget & $1.212\pm0.026$ & $1.619\pm0.019$\\
LSTM, original budget & $1.730\pm0.021$ & $1.801\pm0.020$\\
LSTM, stronger budget & $1.156\pm0.029$ & $1.678\pm0.032$\\
Orthogonal RNN, original budget & $1.790\pm0.008$ & $1.805\pm0.012$\\
Orthogonal RNN, stronger budget & $1.749\pm0.006$ & $1.789\pm0.012$\\
\bottomrule
\end{tabular}

\end{table}

GRU and LSTM validation curves were still improving at the end of the original
120-step budget and improve substantially under the stronger budget. These comparisons support an inductive-bias and sample-efficiency advantage
under the tested protocol, not an inability of generic gated RNNs to integrate.
Extended sweeps and diagnostic plots are supplied as task-level context;
they are not assumptions or proofs of the theoretical results.

\section{Reproducibility}
\label{app:reproducibility}
The supplementary material provides the Python and Lean sources, pinned
dependencies, study configurations, checkpoint registries, recorded
random-number schedules, result tables, and numerical evaluation records.
Figure sources include generation scripts and the documented label-only
correction for the reduced phase-field plot. The supplement's
\path{README_REPRODUCIBILITY.md} maps the scientific studies to their code,
data, and reproduction commands. Historical six-unit records remain separate
from the fresh replication, as described in \cref{app:geometry_selection}.

The executable checks cover equivariance and tangent transport, complete
training-unit and condition accounting, local and finite-horizon bounds,
decoder covariance, and paired task summaries. Unsuccessful pilots and
out-of-regime cases are retained. Random-stream separation is documented in
\cref{app:fresh_replication}, and the higher-dimensional study's operator-norm
correction is explained in \cref{app:noise_geometry}. Source hashes,
dependency versions, complete commands, and dated build records are supplied
with the corresponding code rather than treated as additional empirical
replicates.

\subsection{Formal verification and scope}
\label{app:formal_correspondence}
The supplementary Lean 4 project includes pinned Lean and Mathlib
dependencies, a manuscript-to-declaration mapping, and recorded build and
\texttt{--trust=0} declaration checks. The dependency report lists only
Lean/Mathlib foundational axioms; the source scan finds no project
\texttt{sorry}, \texttt{admit}, or declared \texttt{axiom}.
\Cref{tab:formal_scope} distinguishes the encoded results from mathematical
interfaces and numerical implementation checks.

\begin{table}[ht]
\centering
\scriptsize
\setlength{\tabcolsep}{3pt}
\caption{Formal verification scope. ``External'' means used with explicit
assumptions or a cited theorem, not claimed as a Lean result.}
\label{tab:formal_scope}
\begin{tabularx}{\linewidth}{@{}p{0.32\linewidth}p{0.20\linewidth}X@{}}
\toprule
Manuscript component & Formal status & Notes \\
\midrule
Specified finite graph transport & Lean proved & Unique forward sequence for the supplied initial graph under the denominator condition. \\
Finite center-product bounds & Lean proved & Pointwise lower and upper products in chronological order. \\
Endpoint graph factors & Lean proved & Includes the Euclidean $\sqrt{1+r^2}$ factor and frame endpoints. \\
Subspace-angle control & Lean + manuscript & Lean proves directed-gap sine/arcsine bounds; \cref{lem:directed_gap_principal_angle} identifies the conventional largest principal angle. \\
Leakage scaling & Lean proved & Scalar recurrence gives the stated $O(\epsilon^2/d)$ bound under its smallness assumptions. \\
Symbolic $S^1$ cell equivariance & Lean proved & Invariant coefficients imply exact field and Euler-step equivariance. \\
Group-tangent transport & Lean proved & Differentiation and cocycle telescoping are machine checked. \\
Direct continuous-time neutrality & Lean proved & For a finite-dimensional local action parameterization with uniform tangent bounds. \\
General Lie-orbit construction & External & Smooth action and Lie-exponential orbit curves use standard equivariant dynamics. \\
Full-spectrum multiplicity & External & The at-least-$q$ corollary invokes an applicable multiplicative ergodic theorem. \\
PyTorch/checkpoint correspondence & Executable & Strict loading and numerical residual checks; not a Lean theorem about PyTorch semantics. \\
Stochastic decoder covariance & Manuscript proof & The linear-algebra result is proved in the appendix but not formalized in Lean. \\
Empirical coverage and task claims & Executable & Recomputed from released row-level tables; not theorem evidence. \\
\bottomrule
\end{tabularx}
\end{table}

For the learned $S^1$ family, the formal development connects symbolic
equivariance to tangent transport and zero direct-tangent growth under
noncollapse and uniform action bounds. The continuous-time result uses a
finite-dimensional local action parameterization, without a multiplicative
ergodic theorem. The broader full-spectrum conclusion uses the external
theorem under the assumptions stated in \cref{cor:stationary_cocycle}.

Separate Python checks compare the declared $S^1/T^2$ formulas with stored
checkpoints through strict loading, parameter counts, sampled equivariance,
and analytical tangent transport. The largest discrepancy in the 24 recorded
moving-frame rows was $1.11\times10^{-15}$, within floating-point roundoff.
These checks do not formalize PyTorch execution semantics. The stochastic
decoder analysis, including anisotropic radius allocation, is proved in the
manuscript; the bounded-radial $T^4/T^8$ implementation is tested numerically
and is not covered by the Lean proof.

\subsection{Computational resources}
The fresh replication and higher-dimensional study used a CPU environment
with Python 3.12.14, NumPy 2.4.6, and PyTorch 2.5.1+cpu on an AMD Ryzen 9 9950X
with 16 cores and 125.6 GiB of installed memory. The recorded sum of run-wall
times was 717.63 s for the 100 $T^2$ replication runs and one disjoint pilot,
55.60 s for the four initial higher-dimensional pilots, and 4,511.03 s for
the 192 higher-dimensional confirmation runs. The confirmation used one
PyTorch CPU thread. These are sums for the stated run groups, not the total
computational cost of the project; evaluation replays and additional
checks are recorded separately. Peak memory was not instrumented. Full
software pins, runtime records, and reproduction commands are included in the
supplement.

\end{document}